\documentclass[lettersize,journal]{IEEEtran}

\usepackage{siunitx}
\usepackage{relsize}
\usepackage{ifthen}
\usepackage[colorinlistoftodos]{todonotes}

\usepackage[caption=false]{subfig}

\usepackage{graphics} 
\usepackage{rotating}
\usepackage{color}
\usepackage{enumerate}
\usepackage[T1]{fontenc}
\usepackage{psfrag}
\usepackage{epsfig} 
\usepackage{booktabs}
\usepackage{graphicx,url}
\usepackage{multirow}
\usepackage{array}
\usepackage{latexsym}
\usepackage{amsfonts}
\usepackage{amsmath}
\usepackage{amssymb}
\usepackage{mathtools}
\usepackage{amsthm}
\usepackage{xstring}
\usepackage{algorithm} 
\usepackage[noEnd=true, indLines=true, commentColor=NavyBlue, endLComment=~, beginComment=//~]{algpseudocodex} 
\usepackage{xpatch}
\usepackage{multirow}
\usepackage{xcolor}
\usepackage[dvipsnames]{xcolor}
\usepackage{prettyref}
\usepackage{flexisym}
\usepackage{bigdelim}
\usepackage{listings}

\usepackage{enumitem}
\usepackage{xspace}
\usepackage{bm}
\graphicspath{{./figures/}}
\usepackage{tikz}
\usetikzlibrary{matrix,calc}
\usepackage[acronym]{glossaries}
\usepackage[bookmarks=true]{hyperref}
\usepackage[capitalize]{cleveref}
\usepackage{subcaption}
\usepackage{aliascnt}

\usepackage{mdwlist}

\let\itemize\undefined
\makecompactlist{itemize}{stditemize}

\usepackage[markup=underlined]{changes} 
\definechangesauthor[name=Cathy, color=teal]{Cathy}
\definechangesauthor[name=Federico, color=blue]{Federico}
\definechangesauthor[name=Jiarui, color=purple]{Jiarui}
\definechangesauthor[name=Alessandro, color=orange]{Alessandro}

\newrefformat{prob}{Problem\,\ref{#1}}
\newrefformat{def}{Definition\,\ref{#1}}
\newrefformat{sec}{Section\,\ref{#1}}
\newrefformat{sub}{Section\,\ref{#1}}
\newrefformat{prop}{Proposition\,\ref{#1}}
\newrefformat{app}{Appendix\,\ref{#1}}
\newrefformat{alg}{Algorithm\,\ref{#1}}
\newrefformat{cor}{Corollary\,\ref{#1}}
\newrefformat{thm}{Theorem\,\ref{#1}}
\newrefformat{lem}{Lemma\,\ref{#1}}
\newrefformat{fig}{Fig.\,\ref{#1}}
\newrefformat{tab}{Table\,\ref{#1}}

\theoremstyle{definition}

\theoremstyle{definition}
\newtheorem{problem}{Problem} 
\crefname{problem}{Problem}{Problems}

\newaliascnt{trule}{theorem}

\aliascntresetthe{trule}
\crefname{trule}{Rule}{Rules}

\newaliascnt{corollary}{theorem}

\aliascntresetthe{corollary}
\crefname{corollary}{Corollary}{Corollaries}

\newaliascnt{conjecture}{theorem}

\aliascntresetthe{conjecture}
\crefname{conjecture}{Conjecture}{Conjectures}

\newaliascnt{lemma}{theorem}
\newtheorem{lemma}[lemma]{Lemma}
\aliascntresetthe{lemma}
\crefname{lemma}{Lemma}{Lemmas}

\newaliascnt{assumption}{theorem}

\aliascntresetthe{assumption}
\crefname{assumption}{Assumption}{Assumptions}

\newaliascnt{definition}{theorem}
\newtheorem{definition}[definition]{Definition}
\aliascntresetthe{definition}
\crefname{definition}{Definition}{Definitions}

\newaliascnt{proposition}{theorem}

\aliascntresetthe{proposition}
\crefname{proposition}{Proposition}{Propositions}

\newaliascnt{remark}{theorem}
\newtheorem{remark}[remark]{Remark}
\aliascntresetthe{remark}
\crefname{remark}{Remark}{Remarks}

\newaliascnt{example}{theorem}
\newtheorem{example}[example]{Example}
\aliascntresetthe{example}
\crefname{example}{Example}{Examples}

\newcommand{\bdmath}{\begin{dmath}}
\newcommand{\edmath}{\end{dmath}}
\newcommand{\beq}{\begin{equation}}
\newcommand{\eeq}{\end{equation}}
\newcommand{\bdm}{\begin{displaymath}}
\newcommand{\edm}{\end{displaymath}}
\newcommand{\bea}{\begin{eqnarray}}
\newcommand{\eea}{\end{eqnarray}}
\newcommand{\beal}{\beq \begin{array}{ll}}
\newcommand{\eeal}{\end{array} \eeq}
\newcommand{\beas}{\begin{eqnarray*}}
\newcommand{\eeas}{\end{eqnarray*}}
\newcommand{\ba}{\begin{array}}
\newcommand{\ea}{\end{array}}
\newcommand{\bit}{\begin{itemize}}
\newcommand{\eit}{\end{itemize}}
\newcommand{\ben}{\begin{enumerate}}
\newcommand{\een}{\end{enumerate}}

\newcommand{\calP}{{\cal P}}

\renewcommand{\boldsymbol}[1]{{\bm #1}}

\newcommand{\hide}[1]{}

\newcommand{\hiddenText}{{\color{gray} hidden text.}}
\newcommand{\hideWithText}[1]{\hiddenText}

\DeclareMathOperator*{\argmin}{arg\,min}

\newcommand{\vu}{\boldsymbol{u}}

\newcommand{\vtau}{\boldsymbol{\tau}}

\newcommand{\blue}[1]{{\color{blue}#1}}

\newcommand{\linkToPdf}[1]{\href{#1}{\blue{(pdf)}}}
\newcommand{\linkToPpt}[1]{\href{#1}{\blue{(ppt)}}}
\newcommand{\linkToCode}[1]{\href{#1}{\blue{(code)}}}
\newcommand{\linkToWeb}[1]{\href{#1}{\blue{(web)}}}
\newcommand{\linkToVideo}[1]{\href{#1}{\blue{(video)}}}
\newcommand{\linkToMedia}[1]{\href{#1}{\blue{(media)}}}
\newcommand{\award}[1]{\xspace} 

\newcommand{\myparagraph}[1]{\noindent\textbf{#1}}
\newcommand*\circled[1]{\tikz[baseline=(char.base)]{
            \node[shape=circle,draw,inner sep=0.5pt] (char) {#1};}}

\newcommand*\anothercircled[1]{\raisebox{.5pt}{\textcircled{\raisebox{-.9pt} {#1}}}}

\newcommand{\traj}{\tau}
\newcommand{\setoftraj}{\vtau}
\newcommand{\throughput}{\alpha}
\newcommand{\setofInstances}{\mathbb{I}}
\newcommand{\candidates}{\mathrm{cand}}
\newcommand{\neighbors}{\mathrm{neigh}}
\newcommand{\reachable}{\mathcal{R}}
\newcommand{\soc}{\boldsymbol{C}}
\newcommand{\trajcost}{\mathcal{C}}
\newcommand{\distance}{\mathrm{dist}}

\newcommand{\cI}{\mathcal{I}}
\newcommand{\cE}{\mathcal{E}}
\newcommand{\cG}{\mathcal{G}}

\newcommand{\bmat}{\left[ \begin{array}}
\newcommand{\emat}{\end{array}\right]}

\newcommand{\tup}[1]{\left( #1\right)}

\newacronym{acr:mapf}{MAPF}{Multi-Agent Path Finding}
\newacronym{acr:fico}{FICO}{Finite-Horizon Closed-Loop Factorization}
\newacronym{acr:soc}{SOC}{Sum of Cost}
\newacronym{acr:mpc}{MPC}{Model Predictive Control}
\newacronym{acr:rl}{RL}{Reinforcement Learning}
\newacronym{acr:ert}{ERT}{Execution Response Time}

\usepackage{ifthen}
\newboolean{anonymous}
\setboolean{anonymous}{false}
\usepackage{changes}
\usepackage{cite}

\renewcommand{\baselinestretch}{0.994}
\begin{document}

\title{ 
FICO: \textbf{FI}nite-Horizon \textbf{C}l\textbf{O}sed-Loop Factorization for\\ Unified Multi-Agent Path Finding
}

\ifthenelse{\boolean{anonymous}}{%
  \author{Anonymous Authors}%
}{
\author{Jiarui Li,
        Alessandro Zanardi,
        Federico Pecora,
        Runyu Zhang,
        and Gioele Zardini

\thanks{Jiarui Li, Runyu Zhang, and Gioele Zardini are with the Laboratory for Information and Decision Systems, Massachusetts Institute of Technology, Cambridge, MA, USA (e-mails: \{jiarui01, runyuzha, gzardini\}@mit.edu).}
\thanks{Alessandro Zanardi is with Embotech AG, 8005 Zürich, Switzerland (e-mail: zanardi@embotech.com).}
\thanks{Federico Pecora is with Amazon Robotics, North Reading, MA, USA (e-mail: fpecora@amazon.com).}
\thanks{This work was supported by Prof. Zardini's grant from the MIT Amazon Science Hub, hosted in the Schwarzman College of Computing.}
}}

\maketitle

\begin{abstract}
\gls{acr:mapf} is a fundamental problem in robotics and AI, yet most existing formulations treat planning and execution separately and address variants of the problem in an ad hoc manner. 
This paper presents a system-level framework for \gls{acr:mapf} that integrates planning and execution, generalizes across variants, and explicitly models uncertainties.
At its core is the \emph{\gls{acr:mapf} system}, a formal model that casts \gls{acr:mapf} as a control design problem encompassing classical and uncertainty-aware formulations.
To solve it, we introduce \emph{\gls{acr:fico}}, a factorization-based algorithm inspired by receding-horizon control that exploits compositional structure for efficient closed-loop operation.
\gls{acr:fico} enables real-time responses---commencing execution within milliseconds---while scaling to thousands of agents and adapting seamlessly to execution-time uncertainties.
Extensive case studies demonstrate that it reduces computation time by up to two orders of magnitude compared with open-loop baselines, while delivering significantly higher throughput under stochastic delays and agent arrivals.
These results establish a principled foundation for analyzing and advancing \gls{acr:mapf} through system-level modeling, factorization, and closed-loop design.
\end{abstract}

\begin{IEEEkeywords}
Multi-Agent Path Finding, system modeling, factorization, receding-horizon control, closed-loop algorithms, uncertainty modeling, scalability, robustness.
\end{IEEEkeywords}

\begin{figure}[t]
    \centering
    \includegraphics[width=\linewidth, trim=18.5cm 1.5cm 17.7cm 1.5cm, clip]{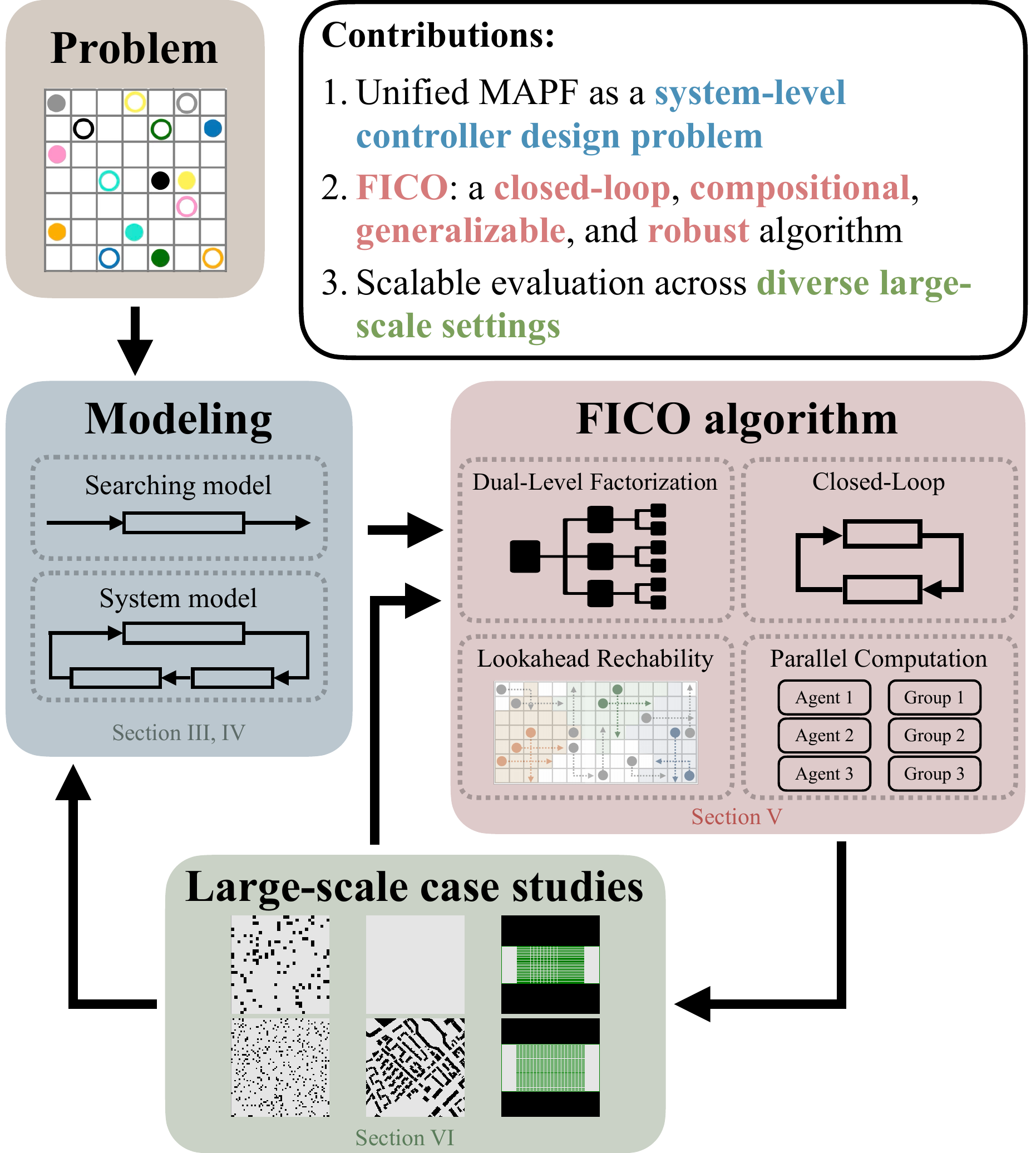}
    \caption{\textbf{Our contributions}: (i) We introduce and formalize the \emph{\gls{acr:mapf} system}, a system-level model that unifies planning and execution, generalizes across \gls{acr:mapf} variants, and naturally incorporates uncertainties, upon which we define the unified \gls{acr:mapf} problem; (ii) We propose the \gls{acr:fico} algorithm, inspired by receding horizon control, to solve the unified problem in an efficient closed-loop manner by applying finite-step lookahead which exploits compositionality; (iii) We present extensive case studies demonstrating that \gls{acr:fico} achieves competitive \emph{performance}, outstanding \emph{efficiency}, and strong \emph{robustness} to diverse execution uncertainties.}
    \label{fig:f1}
\end{figure} 

\section{Introduction}
\IEEEPARstart{M}{odern} logistics increasingly relies on large-scale automated warehouses and fulfillment centers, where thousands of mobile robots must \emph{coordinate} in shared spaces to transport inventory, restock shelves, and fulfill customer orders.
A landmark moment in this transformation was the deployment of robots by Kiva Systems, later acquired by Amazon, which demonstrated that fleets of mobile robots could replace rigid conveyor-based automation with flexible, software-driven coordination~\cite{d2012guest}.
This paradigm shift unlocked unprecedented throughput, scalability, and adaptability, and has since set the standard for the logistics industry.
Today, robotized fulfillment centers are central to global supply chains, e-commerce, and manufacturing.
At the same time, they expose fundamental algorithmic challenges: robots must navigate dense environments, avoid collisions, and complete tasks efficiently under tight deadlines, uncertainties, and continual goal reassignment~\cite{wurman2008coordinating, standley2010finding}.

At the core of these challenges lies the \gls{acr:mapf} problem, a canonical abstraction for multi-robot coordination.
Intuitively, \gls{acr:mapf} asks how to route a fleet of robots from their start locations to their destinations through a shared network of aisles, corridors, or vertices on a graph, ensuring that no two robots ever occupy the same location or traverse the same edge in opposite directions at the same time.
The difficulty comes not from moving one robot, but from coordinating \emph{all} robots simultaneously, avoiding deadlocks and congestion while minimizing global objectives such as overall travel time or energy.
\gls{acr:mapf} has received sustained attention within the robotics, AI, and operations research communities because it captures the combinatorial difficulty of coordinating many agents in discrete and congested environments while retaining a clean mathematical structure~\cite{stern2019mapf}.
Beyond automated warehouses, \gls{acr:mapf} models are directly relevant to other domains of robotics and AI planning, such as manipulation (multi-arm motion planning)~\cite{shaoul2024accelerating}, \gls{acr:rl} (multi-agent coordination under uncertainty)~\cite{wang2025lns2+}, and game-theoretic control~\cite{paul2022multi}.
Yet, despite its wide applicability, the classical formulation of \gls{acr:mapf} has been predominantly cast as a search-based problem: given a set of start and goal locations, compute complete trajectories for all agents that avoid conflicts.
This perspective has enabled remarkable progress, yielding algorithms with strong optimality guarantees and impressive practical performance.
At the same time, it abstracts away the interactive control loop that arises in real deployments, where robots must execute actions in real time, respond to disturbances, and adapt to changing task assignments.

As a result, many existing algorithms, including CBS~\cite{sharon2015conflict}, BCP~\cite{lam2022branch}, ICTS~\cite{sharon2013increasing}, EECBS~\cite{li2021eecbs}, M*~\cite{wagner2015subdimensional}, LNS~\cite{li2022mapf}, and LaCAM~\cite{okumura2023lacam}, are naturally designed in an \emph{open-loop} manner: they produce globally consistent solutions before execution begins. 
While highly effective in static benchmarks, this formulation can create friction in time-critical or uncertain environments, where delaying execution or repeatedly replanning is costly.
To mitigate this, the community has developed a wide range of execution-aware strategies, from post-processing policies for handling delays~\cite{ma2017multi,kottinger2024introducing,su2024bidirectional,vsvancara2019online} to lifelong extensions~\cite{ma2017lifelong, li2021lifelong, zhang2024guidance} and real-time variants that explicitly incorporate execution feedback~\cite{zhang2024planning,liang2025real}.

What is missing, however, is a unifying lens that ties these rich and diverse directions together, under which generalizable and robust \emph{closed-loop} algorithms can be created. 
In contrast, other planning domains, including multi-agent \gls{acr:rl}~\cite{zhang2021multi} and manipulator control~\cite{levine2016end}, have long adopted a control-theoretic perspective in which planning and execution are unified within a feedback loop.
Bringing a similar systematic viewpoint to \gls{acr:mapf} can help contextualize existing methods within a common framework, clarify the role of execution, and open the door to principled treatments of real-time feedback, dynamic replanning, and uncertainties such as stochastic delays or variable fleet sizes.

Moreover, the ``closed-loopness'' enables the corresponding algorithms to leverage \emph{finite-step lookahead}, inspired by receding horizon control, to improve computational efficiency compared with open-loop algorithms, which by definition plan for the entire horizon until termination while maintaining competitive solution quality. On top of that, the finite-step lookahead further opens the door for exploiting compositionality in the sense that agents, although coupled with each other in a long horizon, can be disentangled in a short horizon when applying finite-step lookahead, therefore enabling efficient factorization, which can dramatically tear down the scale of the original problem, especially when combined with modern parallel computation techniques. 

\subsection{Statement of Contribution}
This paper advances the state of the art in \gls{acr:mapf} by introducing a unified execution-aware perspective and a scalable closed-loop algorithm.
Specifically, we make three key contributions.

\myparagraph{System-level modeling of \gls{acr:mapf}} --
We introduce the notion of a \gls{acr:mapf} system, a system-level model that integrates planning and execution within a single feedback loop.
This formulation generalizes classical one-shot \gls{acr:mapf}~\cite{stern2019mapf}, lifelong \gls{acr:mapf}~\cite{ma2017lifelong}, and execution-aware extensions, providing a principled language for analyzing how plans interact with real-time execution and uncertainties.
The model explicitly distinguishes the roles of the controller (planner), actuator (execution), and environment, enabling systematic treatment of uncertainty.
This framework unifies previously fragmented \gls{acr:mapf} variants and situates them as special cases of a single, coherent controller synthesizing problem.

\myparagraph{Closed-loop \gls{acr:fico} algorithm} -- 
Building on this model, we develop \gls{acr:fico}, a closed-loop algorithm inspired by receding-horizon control.
\gls{acr:fico} introduces a finite-horizon dual-level factorization scheme that decomposes large problems into hierarchically organized subproblems based on finite-horizon conflicts and spatial-temporal reachability, which can be solved in parallel and composed without conflicts thanks to the factorization-compatible replanning. On top of these, by integrating with a balanced uniform tie-breaker and an efficient disjoint union-find algorithm, this design allows \gls{acr:fico} to i) respond in real time and enable immediate execution onset, ii) maintain competitive solution quality compared to state-of-the-art planners, iii) scale effectively to dense, large-scale instances by exploiting compositionality, and iv) remain robust in the face of execution uncertainties and dynamic environments.

\myparagraph{Comprehensive experimental evaluation} --
We evaluate \gls{acr:fico} across diverse \gls{acr:mapf} settings, including classical one-shot problems, lifelong task assignments, and large-scale warehouse-style benchmarks.
Across all domains, the dual-level factorization consistently improves both efficiency and robustness. 
In large-scale warehouse studies with thousands of agents, \gls{acr:fico} sustains real-time planning---commencing execution within 15\,ms---and reduces computation time by up to $97.7\%$ relative to open-loop baselines. 
Despite this speed, it preserves solution quality: under uncertainty, it delivers up to $35.7\%$ higher throughput than closed-loop PIBT, and improves item delivery by $3.5\times$--$32\times$ over open-loop baselines facing stochastic delays or agent arrivals. 
In lifelong scenarios with dynamically changing goals, \gls{acr:fico} achieves $59.2\%$ throughput gains over strong closed-loop baselines. 
An ablation study further highlights the complementary role of horizon length, factorization, and planner choice. 
Together, these results position \gls{acr:fico} on a new Pareto frontier, balancing speed, scalability, and robustness for real-time multi-robot coordination.

\section{Related Work}
Research on \gls{acr:mapf} spans a wide range of formulations and algorithmic strategies, reflecting its central role in multi-robot coordination. 
Classical approaches model the problem as a search task and have produced a rich body of algorithms with strong theoretical guarantees. 
Over time, these methods have been extended to settings that more closely resemble real-world deployments, including lifelong variants, execution-aware planning, and approaches explicitly addressing uncertainty. 
Complementary lines of work explore factorization and parallelization to improve scalability. This section reviews these developments, positioning them within the broader context and highlighting the gaps that motivate our contributions.
\subsection{One-shot \gls{acr:mapf}}
The \emph{one-shot} formulation is the classical version of \gls{acr:mapf}, where each agent is assigned a single start and goal location.
It can be cast as a search problem in a high-dimensional joint state space, where a state encodes the positions of all~$N$ agents and transitions correspond to their simultaneous movements.
This formulation quickly leads to combinatorial explosion: for instance, in a four-connected grid where each agent has five possible actions (four moves and waiting), the branching factor reaches~$5^N$, rendering standard search techniques such as DFS, BFS, Dijkstra's algorithm, or A* intractable at scale~\cite{lavalle2006planning,felner2017search}.

To mitigate this, a wide range of specialized algorithms have been developed.
Optimal or bounded-suboptimal methods such as CBS~\cite{sharon2015conflict}, BCP~\cite{lam2022branch}, ICTS~\cite{sharon2013increasing}, ICBS~\cite{boyarski2015icbs}, EECBS~\cite{li2021eecbs}, and M*~\cite{wagner2015subdimensional} reduce the search space but still suffer from exponential scaling and the inherent NP-hardness of the problem~\cite{yu2013planning}.
Faster alternatives, including MAPF-LNS2~\cite{li2022mapf} and LaCAM~\cite{okumura2023lacam}, as well as anytime variants such as MAPF-LNS~\cite{li2021anytime}, LaCAM*~\cite{okumura2023improving}, and engineered LaCAM*~\cite{okumura2023engineering}, improve practical feasibility but remain fundamentally limited by offline computation, often producing solutions of uneven quality~\cite{shen2023tracking}.

\subsection{Lifelong \gls{acr:mapf}} \label{sec:related-work-lifelong-mapf}
The \emph{lifelong} variant of \gls{acr:mapf}~\cite{ma2017lifelong} extends the one-shot formulation by assigning agents new goals as they complete previous ones, capturing realistic scenarios such as automated warehouses where robots must handle continuous pickup-and-delivery tasks. 
While some works assume that future tasks are known in advance~\cite{nguyen2019generalized}, most consider the more realistic case where tasks are revealed online.
The online nature of goal assignment, however, clashes with the offline, search-based framework of classical \gls{acr:mapf}.
A common strategy is to decompose the lifelong problem into a sequence of one-shot instances and repeatedly invoke a one-shot planner.
This effectively turns the problem into a closed-loop process, but repeated replanning is computationally expensive and difficult to scale, even when search trees~\cite{wan2018lifelong} or trajectories~\cite{vsvancara2019online} are reused.
Windowed planning~\cite{li2021lifelong} alleviates this by limiting replanning to a fixed lookahead horizon~$w$ and updating every~$h\leq w$ steps, improving efficiency at the cost of added hyperparameters.
More recently, \emph{guidance graphs} have been proposed to mitigate congestion and improve throughput~\cite{zhang2024guidance, chen2024traffic, zang2025online}, though their synthesis requires computationally intensive and sample-demanding pretraining prior to deployment, which may limit their practicality in real-time applications. 

An alternative direction leverages inherently closed-loop algorithms such as PIBT~\cite{okumura2022priority,okumura2025lightweight}, which generate agents' moves sequentially without precomputing entire trajectories.
Although typically yielding lower solution quality, PIBT's efficiency and generality have made it a common building block within other refined algorithmic designs and frameworks, including the \gls{acr:fico} algorithm proposed in this paper.

\subsection{Execution in \gls{acr:mapf}} \label{sec:related-work-execution}

In \gls{acr:mapf}, \emph{execution} refers to the process of carrying out the planned trajectories.
While lifelong \gls{acr:mapf} inherently involves execution~\cite{ma2017lifelong}, most formulations leave the execution model itself undefined.
When uncertainties are considered, execution is typically handled as a \emph{reactive post-processing step} applied to a precomputed solution.
This component has appeared under various names in the literature: plan-execution policy~\cite{ma2017lifelong}, executor~\cite{zhang2024planning, berndt2023receding,zhang2025concurrent}, execution policy~\cite{atzmon2018robust,atzmon2020robust}, or simply a sequence of \emph{wait} and \emph{move} actions~\cite{liu2024multi}.
These post-processing schemes activate only when deviations render the original plan infeasible.
The typical remedy is to insert dummy vertices (deliberate waits) into the remaining trajectories, reestablishing feasibility under the assumption that execution is perfect from that point onward.
If further deviations occur, the process repeats.
Such methods are a necessary workaround under the classical search model, where replanning and execution are fundamentally decoupled.

Despite their practicality, these schemes face clear limitations.
They usually rely on strong assumptions about uncertainty (most commonly stochastic delays), are often tailored to specific problem variants, and remain computationally demanding due to the complexity of synthesizing inter-agent dependency structures.
Moreover, lifelong \gls{acr:mapf} lacks a precomputed plan to adjust, making reactive modifications even less natural.
As a result, the field still lacks a unified and systematic framework for reasoning about execution in \gls{acr:mapf}.

\subsection{Uncertainty in \gls{acr:mapf}}

Uncertainty is unavoidable in real-world multi-agent systems, where fleets may scale to thousands of robots and even a single failure can propagate widely.
Because \gls{acr:mapf} is closely tied to such applications, handling uncertainty is a central challenge.

Uncertainties can be broadly categorized as \emph{external} (i.e., arising from the environment) and \emph{internal} (i.e., arising from actuation).
A well-studied external source is online goal reassignment, captured by the lifelong \gls{acr:mapf} formulation~\cite{ma2017lifelong}.
Other environmental uncertainties, such as the online addition of new agents, have received less attention and are typically addressed only through costly full replanning~\cite{vsvancara2019online,honig2019persistent}.

Internal uncertainties often manifest as imperfect actuation, with stochastic delays serving as the canonical example.
Robust planning approaches compute~$k$-robust one-shot solutions that tolerate bounded delays~\cite{atzmon2018robust,atzmon2020robust,chen2021symmetry}, but these require heavy upfront computation and lack flexibility. 
More commonly, reactive post-processing is employed: feasibility is restored by inserting waits (dummy vertices) based on inter-agent dependency structures such as Temporal Plan Graphs (TPG) or Action Dependency Graphs (ADG)~\cite{honig2016multi,ma2017multi}.
Extensions such as STPG~\cite{berndt2023receding} and its accelerations~\cite{feng2024real,jiang2025speedup} reduce unnecessary waiting, while alternative dependency structures~\cite{kottinger2024introducing,liu2024multi,su2024bidirectional} improve specific cases.
However, all of these methods require a precomputed solution, and thus apply only to one-shot \gls{acr:mapf}.
To extend delay-handling to the lifelong setting, recent works decompose it into repeated one-shot subproblems and reapply post-processing at each step~\cite{zhuang2025robust,zhang2024planning,zhang2025concurrent}.
While effective in certain scenarios, these methods remain tied to strong assumptions about uncertainty types and struggle to generalize beyond them.

\subsection{Factorization and Parallel Computation}
A central challenge in \gls{acr:mapf} is the curse of dimensionality: as the number of agents grows, the joint search space expands exponentially.
Two main strategies have been explored to mitigate this: \emph{factorization} and \emph{parallelization}, both of which aim to exploit problem structure and computational resources.

Factorization leverages \emph{compositionality}, i.e., the principle that ``the solution of the composition of problems can be approximated by the composition of the solution of problems''.
If subgroups of agents can be planned independently, their solutions may be composed into a globally feasible plan, often preserving performance guarantees such as optimality.
In \gls{acr:mapf}, bottom-up approaches embody this idea by assuming full independence and iteratively merging agents when conflicts arise~\cite{wagner2015subdimensional, standley2010finding,lee2021parallel, zhang2025dynamic}.
However, because these methods typically consider the \emph{entire time horizon} at once, true independence is rare, leading to frequent replanning in overlapping groups and limited scalability. 

Parallel computation, by contrast, seeks to accelerate search by distributing workload across processors.
While most \gls{acr:mapf} planners remain sequential, recent works have applied parallelism to specific components.
Examples include parallelizing CBS search~\cite{lee2021parallel}, decomposing maps into independently planned regions with coordination across boundaries~\cite{leet2022shard}, and parallelizing refinement strategies in LNS-based planners~\cite{okumura2023engineering,chan2024anytime,jiang2024scaling}.
Despite such advances, the strong coupling among agents, especially in dense environments, makes it difficult to decompose problems into independent units that parallelization can fully exploit.

Together, these approaches highlight both the potential and the limitations of exploiting structure for scalability in MAPF. They motivate the need for new formulations that can systematically combine factorization and parallelism within a unified, execution-aware framework.

\subsection{Summary}
In summary, prior work has established strong foundations for \gls{acr:mapf}, from search-based formulations and lifelong variants to execution-aware post-processing and robustness techniques. 
Approaches based on factorization and parallel computation illustrate promising avenues for scalability, but so far they have remained limited: either restricted to bottom-up independence assumptions or confined to local refinements. 
No existing framework systematically integrates factorization with closed-loop execution. 
This gap motivates our contribution: a unified, execution-aware formulation of \gls{acr:mapf} together with an algorithm that treats factorization as a first-class principle. 
By enabling effective factorization with a receding-horizon design, our approach not only provides computational advantages but also delivers competitive solution qualities and enhances resilience to uncertainties, a key requirement for lifelong execution.

\begin{figure}[t]
    \centering
    \includegraphics[width=\linewidth, trim=10cm 0cm 10cm 0cm, clip]{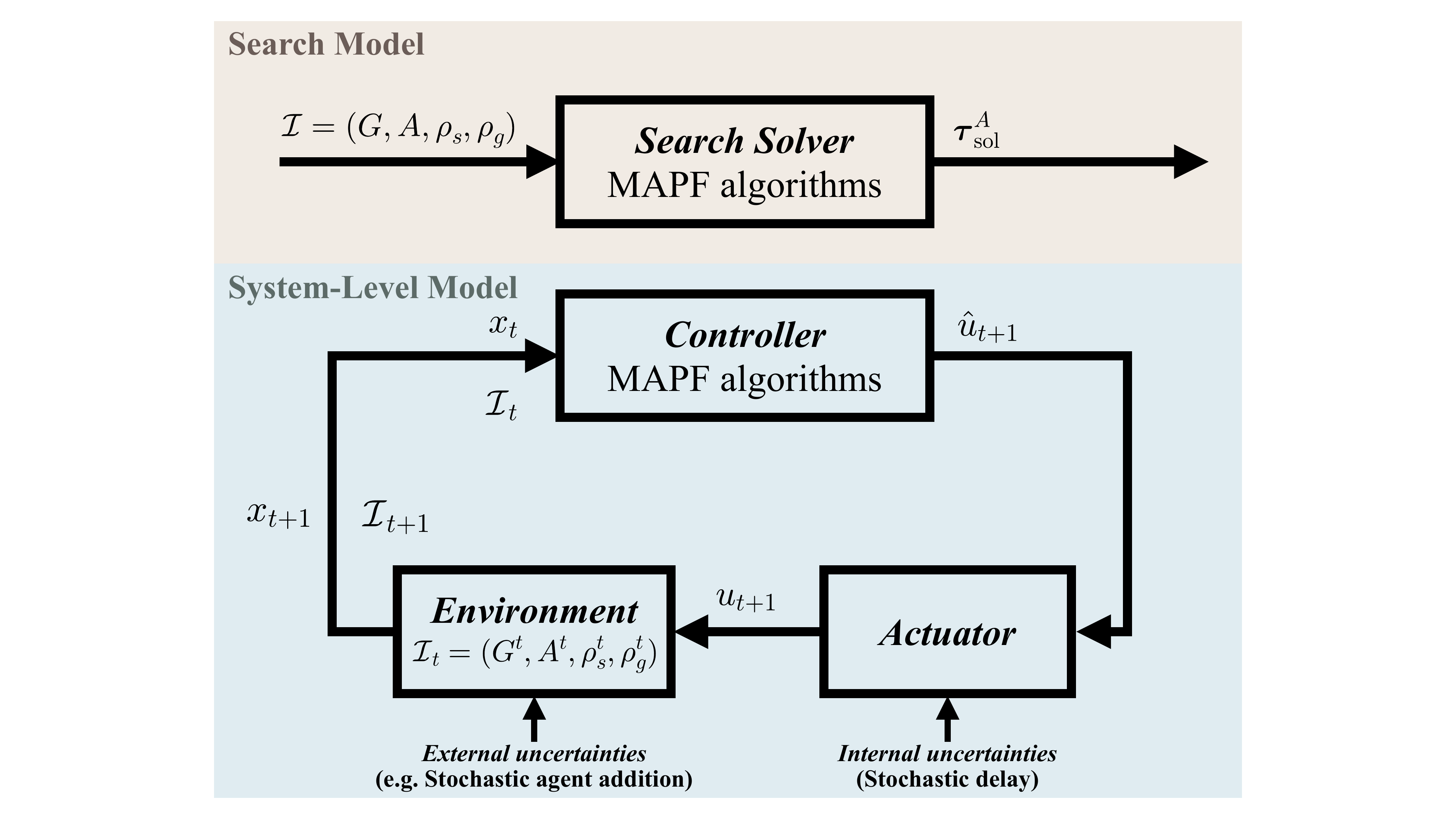}
    \caption{The system-level view of \gls{acr:mapf} system vs. the classical search model.}
    \label{fig:mapf_system_vs_search}
\end{figure}

\section{Classical Search Model of \gls{acr:mapf}}
\label{sec:classical-searching-mapf}

The classical view of \gls{acr:mapf} treats the problem as a search task: given a set of agents with start and goal locations, the objective is to compute complete, conflict-free trajectories for all agents before execution begins.
This perspective, a natural extension of the path-searching problem in computer science, underlies nearly all existing approaches to \gls{acr:mapf}~\cite{stern2019mapf}.

Despite its ubiquity, the searching model has often been described informally or inconsistently in the literature.
To provide a clear foundation for our later developments, this section introduces a uniform set of formal definitions for the search model and its two widely studied: the standard \emph{one-shot} \gls{acr:mapf}, where each agent is assigned a single goal, and the \emph{lifelong} \gls{acr:mapf}~\cite{ma2017lifelong}, where agents receive new goals sequentially as prior ones are completed.

While lifelong \gls{acr:mapf} incorporates the temporal dimension of execution, both formulations remain rooted in the same paradigm: solutions are computed offline as full trajectories, with execution itself excluded from the model. The following subsections formalize these classical definitions to establish a precise baseline for the systematic, execution-aware framework introduced later.

\subsection{One-shot \gls{acr:mapf}}
In the one-shot setting, all agents are given fixed start and goal locations, and the task is to compute a set of conflict-free trajectories.
We formalize this via a \gls{acr:mapf} instance.

\begin{definition}[\gls{acr:mapf} instance]
\label{def:mapf-inst}
An \gls{acr:mapf} \emph{instance} is a tuple~$\cI=\tup{G,A,\rho_s, \rho_g}$ where~$G=\tup{V,E}$ is a directed reflexive graph,~$A=\{a_1, \ldots, a_N\}$ is a set of~$N$ agents, and~$\rho_s\colon A\to V$ and~$\rho_g\colon A\to V$ are maps from each agent~$a\in A$ to a start vertex~$\rho_s(a)=s_a\in V$ and a goal vertex~$\rho_g(a)=g_a\in V$.
\end{definition}

Here, paths and conflicts play a central role in understanding agents' interactions.

\begin{definition}[Path]
\label{def:path}
Let~$G = \tup{V,E}$ be a directed reflexive graph.
A \emph{path} is a finite sequence of vertices~$\pi^G = [v_0, v_1, \ldots, v_T]$, where~$v_t\in V \ \forall t \in \{0, 1, \ldots, T\}$, and consecutive vertices are connected by an edge:~$\forall t \in \{0, 1, \ldots, T-1\}$,~$\tup{v_t, v_{t+1}}\in E$.
The \emph{length} of~$\pi^G$ is~$\ell(\pi^G) \coloneqq T$.
\end{definition}

\begin{definition}[Path conflicts] \label{def:path-conflicts}
Given paths~$\pi^G_i$ and~$\pi^G_j$ of equal length~$T$, we define:
a) \emph{Vertex Conflict}:~$\exists t\in \{0,\ldots,T\}$ such that~$v^i_t = v^j_t$; 
b) \emph{Edge Conflict}:~$\exists t\in \{0,\ldots,T\}$ such that~$v^i_t = v^j_{t-1}$ and $v^i_{t-1} = v^j_t$.
Two paths are \emph{conflict-free} if they exhibit neither vertex conflicts nor edge conflicts.
\end{definition}

In \gls{acr:mapf}, each agent follows a trajectory, obtained by assigning a path to that agent.

\begin{definition}[Trajectory of an agent]
Let~$a_i\in A$ be an agent.
A \emph{trajectory} of~$a_i$, denoted~$\traj^{a_i}=[v^i_0, v^i_1, \ldots, v^i_{T_i}]$, is a path (\cref{def:path}) associated with~$a_i$.
We say that the trajectory~$\traj^{a_i}$ is \emph{valid} if it satisfies~$v^i_0 = \rho_s(a_i)$ and $v^i_{T_i} = \rho_g(a_i)$, where~$\rho_s,g$ denote the start and goal vertex mappings, respectively (\cref{def:mapf-inst}).
Conflicts between trajectories are defined as conflicts between their underlying paths: two trajectories~$\traj^{a_i}$ and~$\traj^{a_j}$ have a vertex (resp. edge) conflict if their underlying paths have a vertex (resp. edge) conflict.
Two agents~$a_i$ and~$a_j$ are \emph{conflict-free} if their assigned trajectories have neither vertex nor edge conflicts at any step.
\end{definition}

\begin{remark}
Some other works in \gls{acr:mapf} use different notion of conflicts, such as cycle conflicts~\cite{stern2019mapf}. 
However, this definition is the most common and suitable to the class of algorithms developed in this work.
\end{remark}

The set of all agents' trajectories constitutes a \emph{solution} to a \gls{acr:mapf} instance.

\begin{definition}[\gls{acr:mapf} solution]
\label{def:mapf-solution}
Let~$\cI=\tup{G,A,\rho_s, \rho_g}$ be a \gls{acr:mapf} instance.
A \emph{solution} to~$\cI$ is a set of valid trajectories~$\vtau_{\text{sol}}^{A} = \{\traj_{\text{sol}}^{a_1}, \ldots, \traj_{\text{sol}}^{a_N}\}$, where $\traj_{\text{sol}}^{a_i}$ is a valid trajectory of agent~$a_i$, such that:
\bit
    \item \emph{Common makespan}:~$\forall a_i\in A$, $\ell(\traj_{\text{sol}}^{a_i}) = T$;\footnote{This condition imposes a common length, or \emph{makespan}, for all trajectories. 
    Since~$G$ is reflexive, an agent that reaches its goal before~$T$ can stay there by waiting ($v^i_t = \rho_g(a_i)$ for all subsequent time steps until~$t=T$), or temporarily vacate the goal to allow passage for other agents before returning.}
    \item \emph{Conflict-freeness}:~$\forall a_i, a_j\in A$, the trajectories $\traj_{\text{sol}}^{a_i}$ and $\traj_{\text{sol}}^{a_j}$ are conflict-free. 
\eit
\end{definition}

\begin{problem}[One-shot \gls{acr:mapf} under the search model]
\label{prob:searching-one-shot-mapf}
Given a \gls{acr:mapf} instance~$\cI$ (\cref{def:mapf-inst}), compute a \gls{acr:mapf} solution~$\vtau_{\mathrm{sol}}^A$ (\cref{def:mapf-solution}). 
\end{problem}

Under the search model, any designed algorithm must output the \emph{entire} solution in a single shot/computation.
No intermediate partial solution (e.g., the first move) can be finalized before the full plan is obtained.
Consequently, feedback from the environment during execution is \emph{not} incorporated into the planning process.
In this context, most existing algorithms generate the complete solution in a batch, often employing backtracking.
As a result, even the first step of execution cannot be initiated until the full plan is computed, an inherent source of delay in time-critical applications~\cite{zhang2024planning}.

In practice, one wants to evaluate the quality of solutions leveraging metrics such as the \emph{makespan} (i.e., the common length of their trajectories, completion time of the last agent), and the \gls{acr:soc} (i.e., the sum of costs for all agents, operational effort, energy expenditure).

\begin{definition}[Metrics for \gls{acr:mapf} solutions] \label{def:oneshot-mapf-metrics}
A \gls{acr:mapf} solution~$\vtau_{\mathrm{sol}}^{A}$ is typically evaluated using:
    \bit
        \item \emph{Makespan}: The length~$T$ of all agents' trajectories in $\vtau_{\mathrm{sol}}^{A}$.
        \item \emph{\gls{acr:soc}}: The cost function is defined as follows:
        \begin{equation*}
        \begin{aligned}
            \soc(\vtau_{\text{sol}}^{A}) & \textstyle = \sum_{a_i\in A} \mathcal{C}(\traj_{\text{sol}}^{a_i}), \\
            \trajcost(\traj_{\text{sol}}^{a_i}) &= \textstyle \sum_{t=0}^{T} I(v^i_t \neq \rho_g(a_i)),
        \end{aligned}
        \end{equation*}
        where~$\trajcost(\traj_{\text{sol}}^{a_i})$ is the cost of the solution trajectory $\traj_{\text{sol}}^{a_i}$, $I(\cdot)$ is the indicator function, and $T$ is the makespan.
    \eit
\end{definition}

\subsection{Lifelong \gls{acr:mapf}}
In the lifelong setting, agents are sequentially assigned new goals as they complete previous ones~\cite{ma2017lifelong}, modeling realistic scenarios such as automated warehouses with continuous pickup-and-delivery tasks.
We now extend the one-shot model to incorporate online goal revelations.

\begin{definition}[Lifelong \gls{acr:mapf} instance]
A \emph{lifelong \gls{acr:mapf} instance} is a tuple~$\cI=\tup{G,A,\rho_s, {\rho_g}, T_{\max}}$ where~$G=\tup{V,E}$ is a directed reflexive graph,~$A=\{a_1, \ldots, a_N\}$ is a set of~$N$ agents,~$\rho_s\colon A\to V$ maps each agent~$a\in A$ to a start vertex~$\rho_s(a)=s_a\in V$, and~${\rho_g}\colon A\to V^{K+1}$,~${\rho_g}(a) = [\rho_g^0(a), \ldots, \rho_g^K(a)]$ maps each agent $a\in A$ to a list of $K+1$ goal vertices it is required to visit sequentially. 
$T_{\max}$ is the maximum allowed number of timesteps.
Initially, only~$\rho_g^0(a)$ is known for each~$a\in A$;
Upon reaching~$\rho_g^k(a)$ at time~$t$, the next goal~$\rho_g^{k+1}(a)$ is revealed.\footnote{We assume~$K+1$ exceeds the maximum number of goals that any agent can reach within~$T_{\max}$.}{}
\end{definition}

\begin{definition}[Trajectory in lifelong \gls{acr:mapf}]
Let~$a_i\in A$.
A \emph{trajectory} of~$a_i$, denoted~$\traj^{a_i} = [v^i_0, \ldots, v^i_{T_i}]$, is a path satisfying~$v_0^i=\rho_s(a_i)$.
The trajectory is \emph{valid} if its length is~$T_\max$.
Conflicts are defined as in the one-shot case.
\end{definition}

In lifelong \gls{acr:mapf}, agents move for the full time horizon~$T_{\mathrm{max}}$, potentially completing multiple assignments.
The complete set of trajectories specifies a solution.

\begin{definition}[Lifelong \gls{acr:mapf} solution]
A \emph{solution} to a lifelong \gls{acr:mapf} instance~$\cI$ is a set of valid trajectories $\vtau_{\text{sol}}^{A} = \{\traj_{\text{sol}}^{a_1}, \ldots, \traj_{\text{sol}}^{a_N}\}$, such that~$\forall a_i, a_j\in A$,~$i\neq j$,~$\traj_{\text{sol}}^{a_i}$ and~$\traj_{\text{sol}}^{a_j}$ are conflict-free. 
\end{definition}

The computational problem is defined as follows.

\begin{problem}[Lifelong \gls{acr:mapf} problem under search model]
\label{prob:searching-lifelong-mapf}
Given a lifelong \gls{acr:mapf} instance~$\cI$, compute a solution~$\vtau_{\mathrm{sol}}^A$ satisfying validity and conflict-freeness constraints, where each agent visits its sequence of goals in order.
In the search model, the full solution is computed offline without feedback from the environment, despite goals being revealed online.
\end{problem}

Because goals are revealed during execution, solution quality is not measured solely by arrival at a final goal, but by how many goals are reached over time.

\begin{definition}[Throughput]
 \label{def:throughput}
For agent~$a_i$ with trajectory~$\traj_{\mathrm{sol}}^{a_i}$ and goal list~$\rho_g(a_i)$, let~$t_0,\ldots,t_{k_i-1}$ be the goal attainment times defined recursively by:
\begin{align*}
t_0 &= \min \{t\mid v_t^i=\rho_g^0(a_i)\},\\
t_j&=\min \{t>t_{j-1}\mid v_t^i=\rho_g^j(a_i)\}, \quad j\geq 1.
\end{align*}
If no such~$t_j$ exists, the sequence terminates.
The number~$k_i$ of attained goals is used to define the \emph{throughput}:
\begin{equation*}
\throughput=\frac{1}{T_\max}\sum_{a_i\in A}k_i,
\end{equation*}
i.e., the average number of goals reached per timestep.
\end{definition}

\cref{prob:searching-lifelong-mapf} and \cref{prob:searching-one-shot-mapf} impose the same validity and conflict-freeness requirements, but differ in their termination conditions.
In one-shot \gls{acr:mapf}, the problem concludes once all agents reach their assigned goals, whereas in lifelong \gls{acr:mapf}, agents operate over a fixed horizon with goals revealed online.
Because future goals are not known in advance, lifelong \gls{acr:mapf} cannot be addressed through a purely offline plan; 
some form of feedback is unavoidable.
A common workaround is \emph{windowed searching}~\cite{li2021lifelong}, which decomposes the horizon into a sequence of one-shot subproblems.
This approach introduces feedback implicitly, but remains tied to repeated search-based planning and does not provide a unified framework for integrating planning and execution.

\section{System-Level Model of \gls{acr:mapf}}
\label{sec:systematic-mapf}

The classical search model of \gls{acr:mapf} introduced in the previous section treats planning as an open-loop problem: given an instance, the algorithm outputs a complete set of trajectories before execution begins.
While this abstraction follows naturally from classical search problems~\cite{cormen2022introduction}, it neglects a key reality: trajectories must ultimately be \emph{executed} in dynamic and uncertain environments, where delays, errors, and disturbances directly affect performance.
Moreover, the search model enforces rigid input-output requirements: a fixed instance must yield a full plan.
This formulation captures one-shot \gls{acr:mapf} exactly, but extensions such as lifelong \gls{acr:mapf} or \gls{acr:mapf}-DP can only be expressed as ad hoc variants, often sacrificing rigor and consistency.
What is missing is a unification of these variants within a single, coherent formulation.

To this end, we introduce the notion of \gls{acr:mapf} system, which casts the problem as a dynamical control loop.
Here, planning and execution interact through three components: the \emph{controller} (planner), which embodies the algorithm and issues movement commands; the \emph{actuator}, which executes the commands subject to uncertainty; and the \emph{environment}, which evolves the system state, such as reveals new goals or agents. 
An illustration comparing the classical search model and the system-level model is shown in~\cref{fig:mapf_system_vs_search}.
This system perspective provides a principled foundation for incorporating execution feedback, modeling uncertainty, and unifying diverse \gls{acr:mapf} variants within a single framework.
It also sets the stage for the \emph{unified \gls{acr:mapf} problem} (\cref{def:unified-mapf}) and the development of our factorization-based closed-loop algorithm, \gls{acr:fico}, which systematically exploits this model to achieve scalability, responsiveness, and robustness.

\begin{definition}[\gls{acr:mapf} system] \label{def:mapf-system}
Let~$\cI_t = \tup{G^t, A^t, \rho_s^t, \rho_g^t}$ be the \gls{acr:mapf} instance at time $t$.
A \emph{state} is a map~$x_t\colon \in (V^t)^{A^t}$, i.e.,~$x_t\colon A^t\to V^t$, mapping each agent to its current vertex.
For~$a\in A^t$, we write~$x_t(a)$ for the position of~$a$ at time~$t$.
A (per-agent) \emph{movement} command is a map~$u_t\in (E^t)^{A^t}$, i.e.,~$u_t\colon A^t\to E^t$, with the consistency constraint
\begin{equation*}
\forall a\in A^t:\; u_t(a) = \tup{x_t(a),\, x_{t+1}(a)} \in E^t,
\end{equation*}
for the executed next state~$x_{t+1}$.
Because~$G^t$ is reflexive, waiting is captured by edges~$\tup{v,v}\in E^t$.
Furthermore, a \gls{acr:mapf} system features the following components:
\begin{itemize}
    \item \emph{Controller}: produces planned movements~$u_t\in (E^t)^{A^t}$.
    If open-loop,~$\hat{u}_t=g(\cI_t,t)$.
    If closed-loop,~$\hat{u}_t=g(x_t,\cI_t,t)$.
    \item \emph{Actuator}: realizes the plan under uncertainties~$\omega_t$,
     \begin{equation*}
  u_t = h(\hat u_t, x_t, \cI_t; \omega_t) \in (E^t)^{A^t}.
  \end{equation*}
  A perfect actuator satisfies~$h(\hat u_t, x_t, \cI_t; \omega_t)\equiv \hat u_t$.
  \item \emph{Environment}: updates the state and possibly the instance:
  \begin{equation*}
  \tup{x_{t+1}, \cI_{t+1}}=f(u_t, \cI_t). 
  \end{equation*}
\end{itemize}
At~$t=0$,~$\cI_0=\tup{G^0,A^0,\rho_s^0,\rho_g^0}$ and~$x_0(a)=\rho_s^0(a) \ \forall a\in A^0.$
\end{definition}

\subsection{Controller}
In the \gls{acr:mapf} system, the \emph{controller} corresponds to the decision-making algorithm that plans agent movements.
Following the taxonomy in~\cite{zhang2024planning}, exisiting \gls{acr:mapf} algorithms can be broadly classified as either \emph{open-} or \emph{closed-loop}.
In the former case, the controller computes the entire plan in advance, without leveraging intermediate state feedback during execution. This is the dominant approach in one-shot \gls{acr:mapf}.
In the latter case, the controller computes movements incrementally, incorporating the current state at each decision step.
Examples include PIBT~\cite{okumura2022priority} and its variants~\cite{okumura2025lightweight, gandotra2025anytime}, which aims only to resolve the immediate next state and action at each time step, yielding greater computational efficiency. This property naturally enables closed-loop adaptation, since each action is computed from the observed current state. Such closed-loop methods are suitable for both one-shot and lifelong settings because of their ability to adaptively select actions from real-time feedback. We now formally define the controller, in both closed-loop and open-loop forms.

\begin{definition}[Controller]
Let~$x_t\in (V^t)^{A^t}$ be the state and~$\cI_t\in \setofInstances$ the instance at time~$t$.
A \emph{controller} is a map
\begin{equation*}
g_t\colon (V^t)^{A^t}\times \setofInstances \to (E^t)^{A^t}
\end{equation*}
that outputs a \emph{planned movement}~$\hat{u}_t=g_t(x_t,\cI_t)$, where~$\hat{u}_t(a)=\tup{x_t(a),\hat{v}_{t+1}(a)}\in E^t$ for each~$a\in A^t$.
If \emph{open-loop},~$g_t$ does not depend on~$x_t$ except at~$t=0$.
All~$\hat{u}_t$ are derived from a precomputed plan:~$\hat u_t = g(\cI_0, t).$
If \emph{closed-loop},~$g_t$ can access~$x_t$ at each step:~$\hat u_t = g(x_t, \cI_t, t).$
\end{definition}

\begin{remark}[Open- vs. closed-loop]
Open-loop algorithms can be deployed in a closed-loop fashion by \emph{replanning} at each step, analogous to \gls{acr:mpc}, which executes only the first move of a plan before re-solving.
However, this typically incurs high computational cost, as each replanning step solves a full \gls{acr:mapf} instance from scratch.
Similarly, the \emph{windowed planning} strategy in lifelong \gls{acr:mapf}~\cite{li2021lifelong} is effectively open-loop when the planning horizon exceeds one step, and equivalent to closed-loop (with the same computational limitations) when the horizon is one step.
\end{remark}

\subsection{Actuator} \label{subsec:actuator}

The \emph{actuator} represents the physical execution layer of the \gls{acr:mapf} system, translating planned movements into actual agent motions in the environment.
In an ideal setting, each agent executes its assigned move exactly as planned at every timestep.
In practice, however, execution may be imperfect due to delays, slippage, hardware limitations, or other disturbances. In the following actuator definition, we explicitly account for imperfect execution, and a concrete example of an imperfect actuator with stochastic delay is provided in \cref{exp:actuator-stochastic-delay}. 

\begin{definition}[Actuator] \label{def:actuator}
Let~$x_t\in (V^t)^{A^t}$ be the current state and~$\hat{u}_t\in (E^t)^{A^t}$ the planned movement produced by the controller.
An \emph{actuator} is a map
\begin{equation*}
h_t : (E^t)^{A^t} \times (V^t)^{A^t} \to (E^t)^{A^t}
\end{equation*}
that outputs the \emph{executed movement}~$u_t=h_t(\hat{u}_t,x_t)$, where~$u_t(a)=\tup{x_t(a),v_{t+1}(a)}\in E^t$ for each~$a\in A^t$.
An \emph{ideal actuator} satisfies~$h_t(\hat{u}_t,x_t)=\hat{u}_t$ for all~$t$, i.e., every planned movement will be executed exactly.
\end{definition}

\subsection{Environment}
The \emph{environment} represents the physical and informational context in which the \gls{acr:mapf} system operates.
It maintains the underlying world state (graph, agent set, start/goal mappings) and evolves over time as agents move.
In the static case, the graph and agent sets remain fixed;
in more general settings, the environment can change dynamically, for example due to {stochastic agent addition} or time-varying graph topology.
At each timestep, the environment receives the executed movement~$u_t$ from the actuator, updates the system state~$x_{t+1}$ accordingly, and produces the updated \gls{acr:mapf} instance~$\cI_{t+1}$ to be used by the controller in the next planning step.

Here we define the environment, and an example of a dynamic environment is provided in \cref{exp:env-agent-addition}. 

\begin{definition}[Environment] \label{def:environment}
Let~$u_t\in (E^t)^{A^t}$ be the executed movement at time~$t$ and~$\cI_t=\tup{G^t,A^t,\rho_s^t,\rho_g^t}$ the current instance.
An \emph{environment} is a map
\begin{equation*}
f_t\colon (E^t)^{A^t}\times \setofInstances_t\to (V^{t+1})^{A^{t+1}}\times \setofInstances_{t+1}
\end{equation*}
that produces the next state~$x_{t+1}$ and the next instance~$\cI_{t+1}$.
In a \emph{perfect static environment}, the environment should satisfy
i) the instance~$\cI_t$ is time-invariant, i.e.,~$\tup{G^{t+1}, A^{t+1}, \rho_s^{t+1}, \rho_g^{t+1}}=\tup{G^t, A^{t}, \rho_s^{t}, \rho_g^{t}}$; 
ii) for any agent~$a \in A^t$ and its action $u_t(a) = (x_t(a), v_{t+1}(a))$, the state transitions evolve deterministically as~$x_{t+1}(a) = v_{t+1}(a)$.

\emph{Lifelong \gls{acr:mapf}} offers an example of dynamic environment where the environment additionally handles online goal reassignment.
If~$x_{t+1}(a_i) = \rho_g^k(a_i)$ for some~$k$, then the goal mapping is updated to~$\rho_g^{t+1}(a_i) = \rho_g^{k+1}(a_i).$

\end{definition}

\subsection{Uncertainty in \gls{acr:mapf} system}

Now that we have formalized the \gls{acr:mapf} system-level model, we can notice that it naturally accommodates imperfect execution and dynamic environments by modifying the actuator and environment transition processes defined in \cref{sec:systematic-mapf}.
We illustrate this modeling capability via two common uncertainty classes.
First, we characterize actuation uncertainty by modeling it as \emph{stochastic delay}, where agents occasionally fail to move as planned, with delays propagating through local dependencies.
Furthermore, we characterize environmental uncertainty by modeling it as \emph{stochastic agent addition}, where new agents enter the environment following a stochastic process, at different times and locations.

\subsubsection{\textbf{Stochastic delay}}
{As mentioned in Section \ref{subsec:actuator}, in practical settings}, executions can be imperfect: some agents can be delayed, producing~$u_t\neq \hat{u}_t$.
In the following, we present a possible case and how to model it.

\begin{definition}[Primary delay set]
At each timestep~$t$, let~$D_t\subseteq A^t$ be the \emph{primary delay set}, where each agent~$a_i\in A^t$ is included independently with probability~$p_{\mathrm{delay}}\in [0,1]$.
\end{definition}

A primary delay may block other agents attempting to move into the delayed agent's current vertex, creating secondary delays that propagate through the team.
This propagation can be captured by a simple graph-based model.

\begin{definition}[Planned movement dependency graph]
Given the current state~$x_t$ and planned movement~$\hat{u}_t$, the \emph{dependency graph}~$G_{\mathrm{dep}}^{t}=(A^t,E_{\mathrm{dep}}^t)$ is defined by:
\begin{equation*}
    E_{\mathrm{dep}}^t=\{\tup{a_i,a_j}\in A^t\times A^t\mid \hat{v}_{t+1}^i=v_t^j\},
\end{equation*}
where~$\hat{v}_{t+1}^i$ is the planned destination of~$a_i$, and~$v_t^j$ is the current position of~$a_j$.
An edge~$\tup{a_i,a_j}$ indicates that~$a_i$'s move depends on~$a_j$ vacating its current vertex.\footnote{Unlike the TPG in~\cite{honig2016multi}, which encodes dependencies between events across all timesteps in a complete solution, our dependency graph represents a single-timestep projection that can be computed online from the planned next-step movements of all agents.}

\end{definition}

\begin{lemma}
The set of stationary agents at time~$t$ is:
\begin{equation*}
    A_{\mathrm{stat}}^t=\{a_i\in A^t\mid \exists a_j\in D_t \text{ s.t. }a_i\to a_j \in G_{\mathrm{dep}}^t\},
\end{equation*}
where $a_i \rightarrow a_j$ denotes that $a_j$ is reachable from $a_i$ via a directed path in $G_{\mathrm{dep}}^t$.
\end{lemma}
\begin{proof}
A primary delay in~$D_t$ blocks its incoming neighbors in~$G_{\mathrm{dep}}^t$, which recursively blocks agents reachable from it.
\end{proof}

\begin{example}[Actuator with stochastic delay]
\label{exp:actuator-stochastic-delay}
Given $D_t$ and~$A_{\mathrm{stat}}^t$, the actual movement is:
\begin{equation*}
    v_{t+1}^i=\begin{cases}
        v_t^i,&a_i\in A_{\mathrm{stat}}^t\\
        \hat{v}_{t+1}^i,&\text{otherwise}.
    \end{cases}
\end{equation*}
This defines the actuator function~$u_t=h(\hat{u}_t,x_t,t,D_t)$ that directly incorporates stochastic delay.
\end{example}

\subsubsection{\textbf{Stochastic agent addition}}
In practice, the environment may change over time due to agents entering or leaving the system (e.g., to recharge or for maintenance purposes).
While removals are trivial to handle, additions introduce new potential conflicts and increase planning complexity.

\begin{example}[Environment transition with stochastic addition] \label{exp:env-agent-addition}
For instance, let~$\mathrm{Add}_t\in \{0,1\}$ be a Bernoulli random variable with~$\mathrm{Pr}(\mathrm{Add}_t=1)=p_{\mathrm{add}}$ and~~$\mathrm{Pr}(\mathrm{Add}_t=0)=1-p_{\mathrm{add}}$.
If~$\mathrm{Add}_t=1$, a new agent~$a_{\mathrm{new}}^t$ is added with start~$v_{\mathrm{n,s}}^t$ and goal $v_{\mathrm{n,g}}^t$.
Let~$\tilde{x}_{t+1}$ be the executed next state after actuation.
The environment update is:
\begin{equation*}
    x_{t+1} = \begin{cases}
            \tilde{x}_{t+1}, &\mathrm{Add}_t = 0 \\
            \tilde{x}_{t+1} \cup \{(a_{\mathrm{n}}^t, v_{\mathrm{n}, \mathrm{s}}^t)\}, & \mathrm{Add}_t = 1
        \end{cases}
\end{equation*}
The corresponding \gls{acr:mapf} instance updates as:
\begin{equation*}
\cI_{t+1}=\begin{cases}
\cI_t, &\text{Add}_t = 0 \\
\hat{\cI}_t,&\text{Add}_t = 1,
\end{cases}
\end{equation*}
where
\begin{equation*}
\hat{\cI}_t=\tup{G^t, A^t \cup \{a_{\text{n}}^t\}, \rho_s^t \cup \{\tup{a_{\text{n}}^t, v_{\text{n}, \text{s}}^t}\}, \rho_g^t \cup \{\tup{a_{\text{n}}^t, v_{\text{n}, \text{g}}^t}\}}.
\end{equation*}
\end{example}

\subsection{Unified \gls{acr:mapf} problem}

The various \gls{acr:mapf} variants introduced in the literature, often with ad hoc differences in assumptions, notation, or scope, can be expressed in a single, coherent framework by viewing them as \emph{control design problems} for the \gls{acr:mapf} system-level model introduced in \cref{sec:systematic-mapf}.
This perspective subsumes traditional one-shot and lifelong \gls{acr:mapf}, and naturally accommodates uncertainty.

\begin{problem}[Unified \gls{acr:mapf} problem]
\label{def:unified-mapf}
Let a \gls{acr:mapf} system at time $t$ be given by:
\begin{itemize}
    \item the current \gls{acr:mapf} instance $\cI_t = (G^t, A^t, \rho_s^t, \rho_g^t)$,
    \item the current state $x_t$,
    \item an actuator that maps planned movements $\hat{u}_t$ and $x_t$ to executed movements $u_t$ (\cref{def:actuator}),
    \item an environment that maps $u_t$ and $\cI_t$ to $\tup{x_{t+1}, \cI_{t+1}}$ (\cref{def:environment}).
\end{itemize}
The \emph{unified \gls{acr:mapf} problem} is the problem of designing a \emph{controller}~$g \colon (x_t, \cI_t, t) \mapsto \hat{u}_t$
that produces a planned movement $\hat{u}_t$ at each timestep $t \ge 0$ until a specified terminal condition holds (e.g.,~$t = T_{\max}$, all agents at goals, or other task-completion criteria).
\end{problem}

\begin{remark}[Solution]
From the executed state sequence~$\{x_t\}_{t=0}^T$, the solution~$\traj_{\mathrm{sol}}^A$ can be reconstructed for performance evaluation:
\begin{equation*}
\traj_{\mathrm{sol}}^{a_i} = [v^i_0, \ldots, v^i_T]
\quad\text{s.t.}\quad
\forall t \quad x_t(a_i)=v^i_t .
\end{equation*}
\end{remark}

\begin{lemma}
\label{lem:unified-mapf-examples}
Under \cref{def:unified-mapf}, the following classical \gls{acr:mapf} problems correspond to specific assumptions on the actuator, environment, and termination condition:
\begin{itemize}
    \item \emph{One-shot \gls{acr:mapf}}: Perfect actuator; static environment; termination when all agents reach their goals.
    \item \emph{Lifelong \gls{acr:mapf}}: Perfect actuator; dynamic environment with online goal updates ($\rho_g$ changes); terminates at~$T_{\max}$.
    \item \emph{\gls{acr:mapf}-DP}~\cite{ma2017lifelong}: Actuator subject to stochastic delays; perfect environment; specialized conflict definitions.
\end{itemize}
\end{lemma}
\begin{proof}
We obtain each classical formulation by fixing the actuator, environment, and termination condition in \cref{def:unified-mapf}.
\paragraph*{One-shot \gls{acr:mapf}}
Assume a \emph{perfect actuator},~$h_t(\hat{u}_t,x_t)=\hat{u}_t$, and a \emph{static environment},~$f_t(u_t,\cI_t)=\tup{x_{t+1},\cI_t}$ with time-invariant~$A,G,\rho_s,\rho_g$.  
Let termination occur when all agents reach their goals.
A classical solver precomputes a conflict-free plan, which the controller~$g$ executes step by step.
Because actuation is perfect and the environment does not change, execution exactly matches the plan, and termination occurs at the common makespan~$T$.  
The reconstructed solution \(\traj_{\mathrm{sol}}^A\) is identical to \cref{def:mapf-solution}, so one-shot MAPF is a special case.

\paragraph*{Lifelong \gls{acr:mapf}}
Keep the \emph{perfect actuator}, but let the \emph{environment} update goals online: if~$x_{t+1}(a)=\rho_g^k(a)$ then~$\rho_g^{t+1}(a)=\rho_g^{k+1}(a)$, while~$A,G$ remain fixed, i.e.,~$G^{t+1}=G^t$ and $A^{t+1}=A^t$.  
Termination is fixed at~$t=T_{\max}$.  
The resulting trajectory~$x_0,\dots,x_{T_{\max}}$ records each agent’s sequence of attained goals, yielding~$k_i$ per agent as in \cref{def:throughput}.  
Throughput~$\throughput=\frac{1}{T_{\max}}\sum_{a_i}k_i$ coincides with the lifelong MAPF metric.

\paragraph*{\gls{acr:mapf}-DP}
Fix a \emph{perfect environment} (static instance) and define an \emph{actuator with stochastic delay}:~$h_t$ may replace the commanded edge with a wait or propagate blocking, exactly as in the MAPF-DP model.
Conflicts are defined accordingly.  
Controllers interacting with this actuator generate trajectories matching MAPF-DP execution semantics.
\end{proof}

\begin{figure*}[t!]
    \centering
    \includegraphics[width=\linewidth, trim=3.5cm 7.5cm 3.5cm 5.5cm, clip]{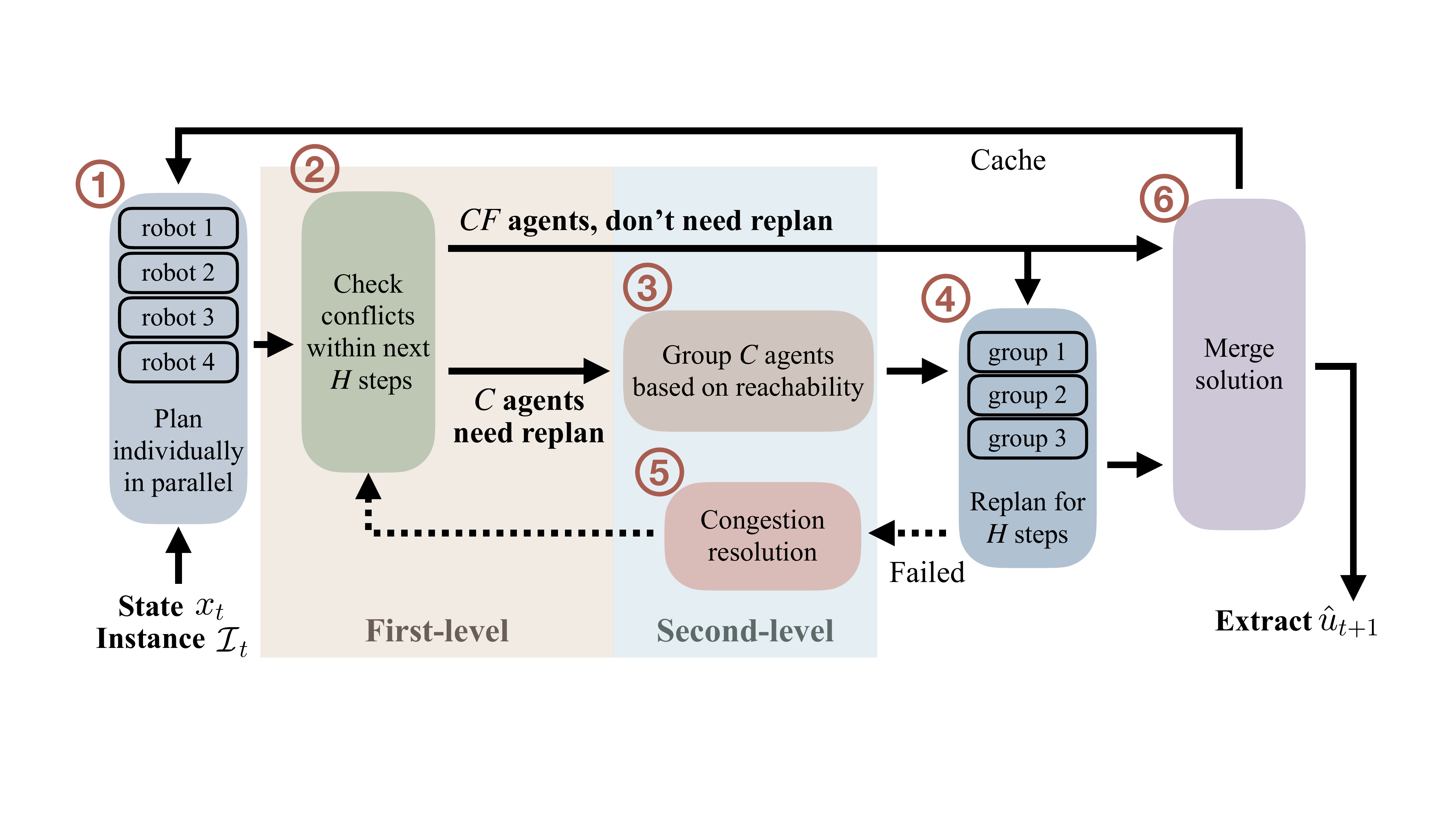}
    \caption{The algorithm plans in a receding-horizon fashion, computing and finalizing robot movements one timestep at a time. This online approach enables immediate execution of each computed step.}
    \label{fig:fico_algo}
\end{figure*}

\section{\gls{acr:fico} Algorithm}
\label{sec:fico-algo}

Since most existing \gls{acr:mapf} algorithms are based on the classical search model (\cref{sec:classical-searching-mapf}), they are inherently \emph{open-loop} and thus poorly suited for dynamic environments with unavoidable uncertainties, such as large-scale warehouses.
Robust \emph{closed-loop} alternatives remain comparatively underexplored.
By reframing \gls{acr:mapf} as a control design problem, we align it with control-theoretic practice, where closed-loop systems are standard for ensuring adaptability~\cite{franklin2010feedback}.
This perspective motivates the central question of this section: can we design computationally efficient, closed-loop algorithms that deliver the benefits of long-horizon planning without its prohibitive computational burden?

The \gls{acr:fico} algorithm addresses this question by importing ideas from receding-horizon control into the \gls{acr:mapf} domain.
Analogous to \gls{acr:mpc}, \gls{acr:fico} repeatedly solves a finite-horizon planning problem, executes only the first action, and then replans using updated state information.
This structure combines the \emph{responsiveness} and \emph{robustness} of closed-loop control with the \emph{performance advantages} of deeper lookahead.
In what follows, we first formalize the notion of lookahead in \gls{acr:mapf} (\cref{sec:lookahead-mapf}) and introduce factorization-compatible trajectories, which enable compositional factorization (\cref{sec:factorization-compatible-traj}).
We then present \gls{acr:fico} (\cref{sec:fico-algo-summary} and \cref{sec:fico-algo-sub}), detailing its main stages and computational properties.

\subsection{Lookahead in \gls{acr:mapf} algorithms}
\label{sec:lookahead-mapf}
A defining feature of any \gls{acr:mapf} solver is its \emph{lookahead horizon}, the extent to which it anticipates future agent interactions when selecting the next action.
Lookahead matters because avoiding conflicts requires reasoning not only about immediate moves but also about how agents' paths will interact over time.

Most established \gls{acr:mapf} algorithms adopt an \emph{infinite-step lookahead}: they compute a complete plan from the current state until all agents reach their goals.
This full-horizon perspective enables strong guarantees such as optimality, completeness, or both~\cite{sharon2015conflict,li2021eecbs, okumura2023lacam, okumura2023improving}, but at a steep computational cost.
Moreover, such algorithms operate as open-loop controllers: once a plan is computed, execution follows it rigidly unless external disruptions force replanning.
In dynamic or uncertain environments, this means frequent replanning from scratch, which is often computationally prohibitive.
At the opposite extreme are \emph{one-step lookahead} methods, such as PIBT~\cite{okumura2022priority}.
Such algorithms consider only the immediate next move of all agents, typically guided by greedy heuristics.
Their advantages are clear: they are naturally closed-loop, lightweight, and reactive.
Yet their short horizon can lead to myopic decisions and poor performance in key metrics such as \gls{acr:soc} and throughput.

The \gls{acr:fico} algorithm is designed to bridge these extremes.
It employs an~$H$-step lookaehad within a closed-loop architecture, directly inspired by receding-horizon control.
At each timestep, i) \gls{acr:fico} plans~$H$ steps into the future, incorporating richer foresight than one-step methods, ii) executes only the first step before replanning, preserving adaptability, and iii) keeps computation tractable compared to full-horizon solvers.
Crucially, finite lookahead is not just a trade-off between foresight and cost: it also enables \emph{efficient hierarchical factorization}.
Because \gls{acr:fico} only needs to guarantee conflict-freeness over~$H$ steps, the global \gls{acr:mapf} problem can be decomposed into smaller, independent subproblems that are solved in parallel and then composed without conflicts.
Correct composition is ensured by requiring that each subproblem's solution remains compatible with higher-priority plans.
\subsection{Factorization-compatible trajectories} 
\label{sec:factorization-compatible-traj}

Factorization has long been recognized as a powerful strategy for tackling otherwise intractable problems: by decomposing a large system into smaller subproblems, one can exploit structure and solve each component more efficiently.
This principle has been applied in diverse areas such as game theory~\cite{zanardi2022factorization} and motion planning~\cite{zanardi2023factorization}.
Yet, despite its promise, it has been underutilized in \gls{acr:mapf}.
The reason lies in the role of lookahead (\cref{sec:lookahead-mapf}).
With infinite-step lookahead, every agent’s trajectory must be planned until its goal is reached, which tightly couples all agents and leaves little room for decomposition.
The~$H$-step lookaehead of \gls{acr:fico} changes this picture.
By limiting the planning horizon, the entanglement across agents can be reduced, exposing the system's inherent compositionality.
In fact, large-scale \gls{acr:mapf} problems typically exhibit \emph{localized congestion}: while a few regions are densely populated and require coordination, the majority of agents move independently and can be planned in isolation, at least temporarily (as shown later in~\cref{tab:agent-number-reduction}).
This structural property makes factorization particularly attractive.
To exploit it, \gls{acr:fico} introduces a dual-level factorization.
In the first level, agents whose trajectories do not interact with others, termed conflict-free (CF) agents, are factored out and handled separately, since their plans can be generated independently without affecting the rest of the system.
The second level focuses only on the remaining conflict (C) agents, whose interactions are nontrivial.
Here, we apply a reachability-based factorization that clusters agents according to their potential conflicts and coordinates their motions jointly.

A key challenge is ensuring that the two levels remain consistent.
Trajectories planned for CF agents must not be invalidated when combined with those of the C agents.
We formalize this requirement through the notion of \emph{factorization-compatible trajectories}:~$H$-step plans for one subset of agents that can be composed with those of another without introducing new conflicts.
This concept provides the theoretical backbone for \gls{acr:fico}'s compositional approach, guaranteeing that independently generated solutions can be assembled into a coherent global plan.

\begin{definition}[Factorization-compatible trajectories]
  \label{def:cond-traj-single}
  Consider a set of agents $A' \subsetneq A$ and an agent $a_i \in A \setminus A'$. 
  Let~$\setoftraj^{A'} = \{\traj^{a_j} \mid a_j \in A'\}$ denote a set of trajectories for agents in~$A'$ given a priori. 
  A trajectory~$\traj^{a_i}$  is \emph{factorization-compatible} given~$\setoftraj^{A'}$, denoted~$\traj^{a_i} \mid \setoftraj^{A_1}$, if it is conflict-free with every trajectory in~$\setoftraj^{A'}$ over the next~$H$ timesteps.
\end{definition}

\begin{definition}[Factorization-compatible trajectory set]
  \label{def:cond-traj-set}
  Let~$A',A''$ be disjoint subsets of~$A$, and let~$\setoftraj^{A'}$ be a given set of trajectories. 
  A set of trajectories~$\setoftraj^{A''}$ is \emph{factorization-compatible} with~$\setoftraj^{A'}$, denoted~$\setoftraj^{A''} \mid \setoftraj^{A'}$, if for each~$a_i \in A''$, its trajectory~$\traj^{a_i} \in \setoftraj^{A''}$ satisfies \cref{def:cond-traj-single}.
\end{definition}

\begin{remark}[Relation to agent priority]
In \gls{acr:fico}'s hierarchical factorization,~$A'$ represents agents assigned higher priority in the iteration at hand.
Their trajectories are fixed for the horizon~$H$, and must remain unaffected by the replanning of~$A''$.
By ensuring~$\setoftraj^{A''}\mid \setoftraj^{A'}$, we guarantee that subproblem solutions can be composed into a globally feasible~$H$-step plan.
This property is a key enabler for \gls{acr:fico} scalability, as it allows the parallel solution of smaller subproblems while preserving correctness. 
\end{remark}

\begin{lemma}[Two-way composition]
\label{lem:twoway}
Let~$A',A''$ be disjoint.
If~$\setoftraj^{A'}$ and~$\setoftraj^{A''}$ are each internally conflict-free over the next~$H$ steps and~$\setoftraj^{A''}\mid \setoftraj^{A'}$, then~$\setoftraj^{A'}\cup \setoftraj^{A''}$ is conflict-free over the next~$H$ steps.
\end{lemma}

\begin{proof}
By definition of factorization-compatibility, no~$a'\in A''$ conflicts with any~$a\in A'$.
Internal conflict-freeness rules out conflicts within~$A',A''$. 
Hence, no pair in~$A'\cup A''$ conflicts over the horizon.
\end{proof}

\begin{lemma}[Composition]
\label{lem:comp-fact}
Let~$\tup{A_1,\ldots,A_m}$ be a partition of~$A$.
Suppose, for each~$k=1,\ldots,m$,~$\setoftraj^{A_k}$ is internally conflict-free.
For any pairs $i,j\in\{1,\ldots,m\}$, $i \neq j$, $\setoftraj^{A_i} \mid \setoftraj^{A_j}$. 
Then,~$\bigcup_{k=1}^{m}\setoftraj^{A_k}$ is conflict-free over the next~$H$ steps.
\end{lemma}
\begin{proof} 
This can be easily proved by induction on~$m$.
The base~$m=2$ is \cref{lem:twoway}.
For the induction step, compose~$\bigcup_{j<k}\setoftraj^{A_j}$ (inductive hypothesis) with~$\setoftraj^{A_k}$ using \cref{lem:twoway}.
\end{proof}

\subsection{Summary of \gls{acr:fico}} \label{sec:fico-algo-summary}

\label{sec:fico-summary}
We now present the \gls{acr:fico} algorithm, a closed-loop \gls{acr:mapf} controller that unifies the~$H$-step receding-horizon principle from \cref{sec:lookahead-mapf} with the factorization-compatible insights from \cref{sec:factorization-compatible-traj}.
At each timestep~$t$, \gls{acr:fico} is given the current \gls{acr:mapf} instance~$\cI$, the current state~$x_t$, and the finite horizon length~$H$.
The goal is to compute the planned movement~$\hat{u}_t$, the first action for every agent, such that i) it is part of a feasible~$H$-step plan from the current state, ii) it respects factorization-compatibility constraints between subgroups, and iii) it can be recomputed at~$t+1$ using updated information, in receding-horizon fashion.

\paragraph{High-level structure}
The algorithm proceeds in six distinct stages (see \cref{fig:fico_algo} and \cref{alg:fico}).
It starts with \emph{parallel individual planning} \circled{1}, where each agent generates its own~$H$-step greedy-optimal path toward its goal.
This leverages a balanced perfect heuristic to match~$A^\star$ optimal path distribution without its full overhead.
The resulting set of trajectories is then analyzed through \emph{first-level factorization} \circled{2}: spatial hashing detects all conflicts within the next~$H$ steps, separating agents into conflict-free and conflicting sets. 
Trajectories of conflict-free agents are fixed as a high priority, in the sense of \cref{def:cond-traj-single}, and will not be altered by subsequent replanning.
Conflicting agents are then refined through \emph{second-level factorization} \circled{3}, where they are grouped according to conflicts in their horizon-limited reachable regions.
This grouping, supported by \cref{lem:reach}, ensures that each group can be replanned independently without creating inter-group conflicts.
Within each group, \emph{factorization-compatible replanning} \circled{4} is performed in parallel for the~$H$-step horizon, explicitly compatible with the fixed trajectories of conflict-free groups and any previously finalized higher-priority groups to preserve factorization-compatibility as in \cref{def:cond-traj-set}.
Such replanning leverages an adapted version of PIBT that incorporates hindrance-aware neighbor selection to improve performance in dense scenarios.
When a group fails to produce a compatible plan, typically due to space being fully occupied by higher-priority trajectories, \gls{acr:fico} triggers \emph{congestion resolution} \circled{5}, expanding the group to include nearby high-priority agents, thereby relaxing constraints and increasing the available maneuvering space.
This expansion is repeated until a feasible plan emerges, and in the worst case the process degenerates to classical PIBT, which can always produce a next step if one exists.
Finally, the algorithm performs \emph{merging and execution} \circled{6}, combining all~$H$-step subplans into a globally conflict-free plan leveraging \cref{lem:comp-fact}.
Only the first step of each trajectory is executed, with the remaining~$H-1$ steps serving as a receding-horizon reference for the next planning cycle.

\subsection{Details of \gls{acr:fico} algorithm} \label{sec:fico-algo-sub}

\makeatletter
\xpatchcmd{\algorithmic}{\itemsep\z@}{\itemsep=0.75pt}{}{}
\makeatother

\begin{algorithm}[t]
\resizebox{\columnwidth}{!}{
\begin{minipage}{\columnwidth}
    \small
    \caption{\gls{acr:fico} algorithm}
    \label{alg:fico}
    \begin{algorithmic}[1]
    \Require {Current instance $\cI_t = \tup{G^t,A^t,\rho_s^t, \rho_g^t}$, current state $x_t$, horizon size $H$} 
    \Ensure Planned movements $\hat{u}_t$
    
    \State $\calP_t = \{\tup{a_i, \mathsf{False}} \mid a_i \in A^t\}$ \Comment{Agent labels for \circled{5}}
    \State $\vtau_{\text{plan}}^{A^t} \gets \emptyset$

    \LComment{\textbf{Step \circled{1}: Parallel individual planning (\cref{alg:individual-optimal-planning})}}
    \State $\hat{\vtau}_{\text{plan}}^{A^t} \gets \Call{IndividualOptimalPlanning}{\cI_t, x_t, H}$ 
  
    \While{$\mathsf{True}$}
        \State $\mathrm{planning\_success} \gets \mathsf{True}$

        \LComment{\textbf{Step \circled{2}: Conflict detection (\cref{alg:conflict-detection})}}
        \State $A_{\mathrm{CF}}^t, A_{\mathrm{C}}^t \gets \Call{DetectConflicts}{\cI_t, \hat{\vtau}_{\text{plan}}^{A^t}, H, \calP_t}$
  
        \State $\vtau_{\text{plan}}^{A_{\mathrm{CF}}^t} \gets \hat{\vtau}_{\text{plan}}^{A_{\mathrm{CF}}^t}$
  
        
        \LComment{\textbf{Step \circled{3}: Reachability grouping (\cref{alg:reachability-grouping})}}
        \State $\mathcal{A}_{C}^t \gets \Call{ReachabilityGrouping}{\cI_t, A_{\mathrm{C}}^t, x_t, H}$

        \LComment{\textbf{Step \circled{4}: Parallel factorization-compatible planning}}
        \ForAll{$A_{\mathrm{C}}^{t,k} \in \mathcal{A}_{C}^t$}
        
            \State $\pi^{t,k} \gets \Call{PrioritiesComputation}{\cI_t, A_{\mathrm{C}}^{t,k}}$
            \State $\widetilde{\vtau}_{\text{plan}}^{A_{\mathrm{C}}^{t,k}} \gets \Call{FactorizationCompatiblePlanning}{\newline 
            \cI_t, A_{\mathrm{C}}^{t,k}, \vtau_{\text{plan}}^{A_{\mathrm{CF}}^t}, x_t, \pi^{t,k}, H}$

            \LComment{\textbf{Step \circled{5}: Congestion resolution}}
            \If{$\widetilde{\vtau}_{\text{plan}}^{A_{\mathrm{C}}^{t,k}} = \emptyset$}
                \State $\mathrm{planning\_success} \gets \mathsf{False}$
                \State $\Call{CongestionResolution}{\cI_t, \vtau_{\text{plan}}^{A^t}, x_t, A_{\mathrm{C}}^t, A_{\mathrm{CF}}^t, \newline 
                H, \calP_t}$ 
                \State \textbf{break}
            \Else
            \State $\vtau_{\text{plan}}^{A_{\mathrm{C}}^{t,k}} \gets \widetilde{\vtau}_{\text{plan}}^{A_{\mathrm{C}}^{t,k}}$
            \EndIf
        \EndFor

        \LComment{\textbf{Step \circled{6}: Merging}}
        \If{$\mathrm{planning\_success}$}
            \State $\hat{\vu}_t \gets \left\{(a_i, \tup{\tau_{\text{plan}}^{a_i}[0], \tau_{\text{plan}}^{a_i}[1]}) \mid a_i \in A^t\right\}$
            
            \State \Return{$\hat{\vu}_t$}
        \EndIf
    \EndWhile
    \end{algorithmic}
\end{minipage}
}
\end{algorithm}

We now detail each stage, linking its design choices to the formal concepts introduced in \cref{sec:lookahead-mapf,sec:factorization-compatible-traj} and to the guarantees they preserve.

\myparagraph{\circled{1} Parallel individual planning} -- 
Given the current instance~$\cI_t$ and state~$x_t$, each agent~$a_i\in A^t$ computes an individual~$H$-step trajectory~$\traj^{a_i}$ from its current vertex~$v_t^i$ toward its goal~$\rho_g^t(a_i)$, ignoring all other agents (Algorithm \ref{alg:individual-optimal-planning}).
Formally, this is the solution of the single-agent~$H$-step planning problem, obtained by restricting the \gls{acr:mapf} instance to~$a_i$ alone, as in \cref{sec:systematic-mapf}.
Because the factorization stages that follow rely on high-quality individual trajectories, we seek cost-optimal solutions for this single-agent subproblem and select one of them uniformly to decrease congestion.
To achieve this efficiently, \gls{acr:fico} leverages a {balanced perfect heuristic}~$\gamma(v,a_i)$ (improved from~\cite{okumura2022priority}) to guide search.
The heuristic gives the exact shortest-path distance from~$v$ to~$\rho_g^t(a_i)$ and is computed lazily via a backward Breadth-First Search (BFS) seeded at the goal, avoiding unnecessary preprocessing (see \cref{alg:perfect_distance_heuristic}). 
When multiple shortest paths exist, it is important to select among them uniformly to reduce the likelihood of conflicts. 
If each agent follows this principle of uniformity, the probability that two agents choose the same vertex or edge at the same time decreases. 
However, when a uniform random tie-breaking rule is applied, where equally optimal edges are selected uniformly at each vertex, such heuristics guarantee individual optimality but do not reproduce the uniform optimal path distribution of exhaustive~$A^\star$ search: some optimal paths may be favored due to tie-breaking biases, as shown in~\cref{fig:uniform_tie_breaker}.
To recover uniformity, we define the \emph{candidate set} for~$v$ as
\begin{equation}
\label{eq:candidates}
\candidates(v) \;=\; \{ v'\in \neighbors(v) \mid \gamma(v',a_i) = \gamma(v,a_i) - 1 \},
\end{equation}
and the \emph{optimal path count}~$c(v,a_i)$ as the number of shortest paths from~$v$ to~$\rho_g^t(a_i)$.
It can be computed efficiently via the lazy dynamic-programming recursion 
\begin{equation}
\label{eq:opt-path-count}
c(v,a_i) \;=\; \textstyle{\sum_{v' \in \candidates(v)} c(v',a_i)},
\end{equation}
using the same backward BFS order as for~$\gamma$ (see \cref{alg:optimal_path_count}).

The following result formalizes the key property of the~$c(\cdot,a_i)$-based tie-breaker: it exactly reproduces the optimal-path distribution of exhaustive~$A^\star$ search.

\begin{lemma}[Optimal path distribution equivalence]
\label{lem:optimal-path-distribution}
Let~$G^t$ be the current graph,~$a_i\in A^t$ an agent, and~$\rho_g^t(a_i)$ its goal.
If at each step the agent chooses its next vertex~$\tilde{v}\in\candidates(v)$ with probability
\begin{equation}
\label{eq:balance-tie-breaker}
\mathrm{Pr}(\tilde{v},a_i) \;=\;
\frac{c(\tilde{v},a_i)}{\sum_{v' \in \candidates(v)} c(v',a_i)},
\end{equation}
then the resulting distribution over complete shortest paths from~$v$ to~$\rho_g^t(a_i)$ is uniform.
\end{lemma}

\begin{proof}
We can prove this by induction on~$d=\gamma(v,a_i)$, the remaining distance to the goal.
Consider the base case~$d=0$: the only path from~$v=\rho_g^t(a_i)$ to itself is the trivial path of length zero, so the distribution is trivially uniform.
Looking at the inductive step, assume the claim holds for all vertices with~$\gamma(\cdot,a_i)<d$.
Consider~$v$ with~$\gamma(v,a_i)=d$.
The probability of selecting a particular shortest path~$\tup{v,v_1,\ldots,v_d}$ is
    $\mathrm{Pr}\tup{v,v_1,\ldots,v_d}=\mathrm{Pr}(v_1\mid v)\cdot \mathrm{Pr}(\tup{v_1,\ldots,v_d)\mid v_1}.$
By the tie-breaking rule,
\begin{equation*}
    \mathrm{Pr}(v_1\mid v)=\frac{c(v_1,a_i)}{\sum_{u\in \candidates(v)}c(u,a_i)}.
\end{equation*}
By definition of~$c(\cdot,a_i)$, we have
\begin{equation*}
    c(v,a_i)=\textstyle \sum_{u\in \candidates(v)}c(u,a_i),
\end{equation*}
so~$\mathrm{Pr}(v_1\mid v)=c(v_1,a_i)/c(v,a_i)$.
By the inductive hypothesis, conditional on choosing~$v_1$, all shortest paths from~$v_1$ to the goal are equally likely, each with probability~$1/c(v_1,a_i)$.
Thus:
\begin{equation*}
    \mathrm{Pr}\tup{v,v_1,\ldots,v_d}=\frac{c(v_1,a_i)}{c(v,a_i)}\cdot \frac{1}{c(v_1,a_i)}=\frac{1}{c(v,a_i)}.
\end{equation*}
Since this expression does not depend on the particular path~$\tup{v,v_1,\ldots,v_d}$, all~$c(v,a_i)$ shortest paths from~$v$ to~$\rho_g^t(a_i)$ are chosen with equal probability.
\end{proof}

\begin{remark}
\cref{lem:optimal-path-distribution} guarantees that, in the absence of conflicts, \gls{acr:fico}'s parallel individual planning produces a uniform distribution over optimal single-agent trajectories, exactly matching the path-selection behavior of exhaustive~$A^\star$ search but at significantly lower computational cost.
This ensures that the subsequent first-level factorization stage~\circled{2} starts from high-quality, unbiased trajectories, maximizing the chance that many agents will be finalized as conflict-free.
\end{remark}

\begin{figure}[tb]
    \centering
    \subfloat[Random tie-breaking]{\includegraphics[width=0.47\linewidth]{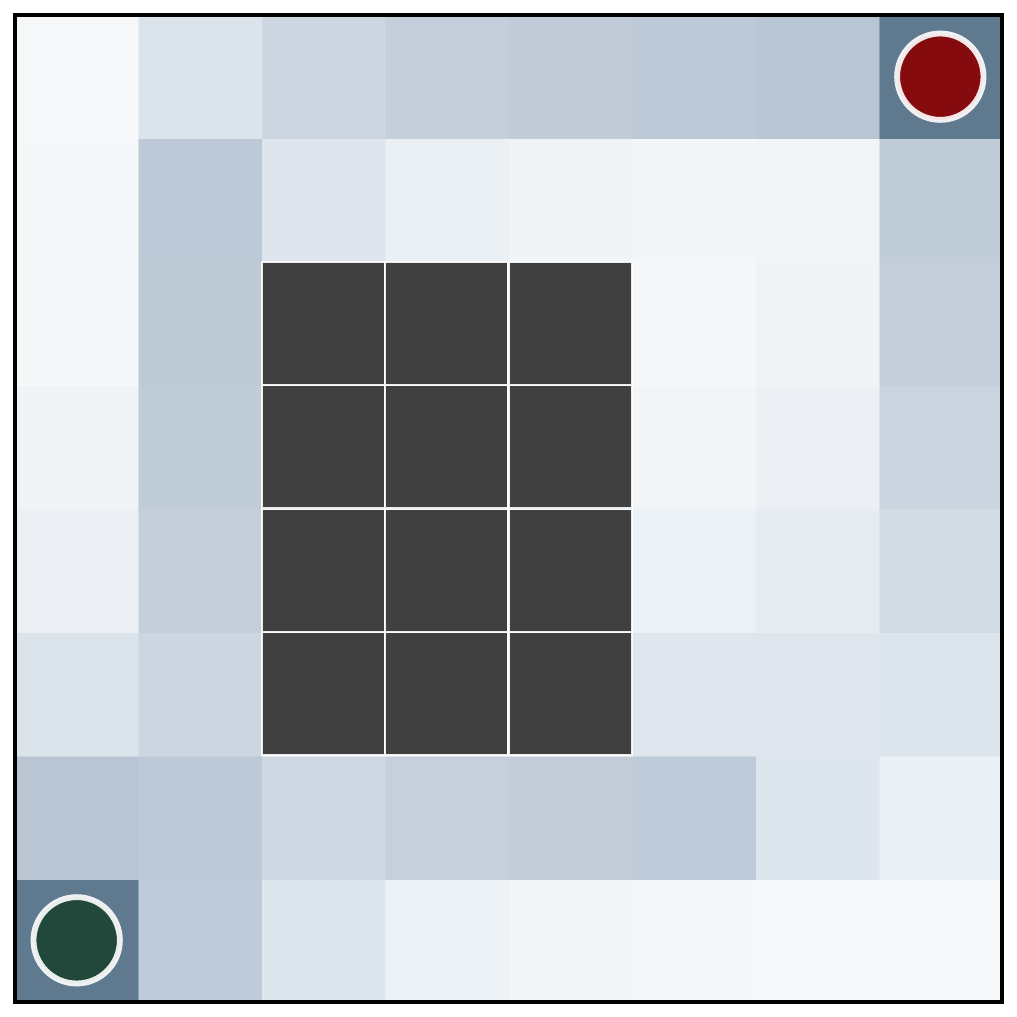}} \hfill
    \subfloat[Balanced tie-breaking \& $A^*$]{\includegraphics[width=0.47\linewidth]{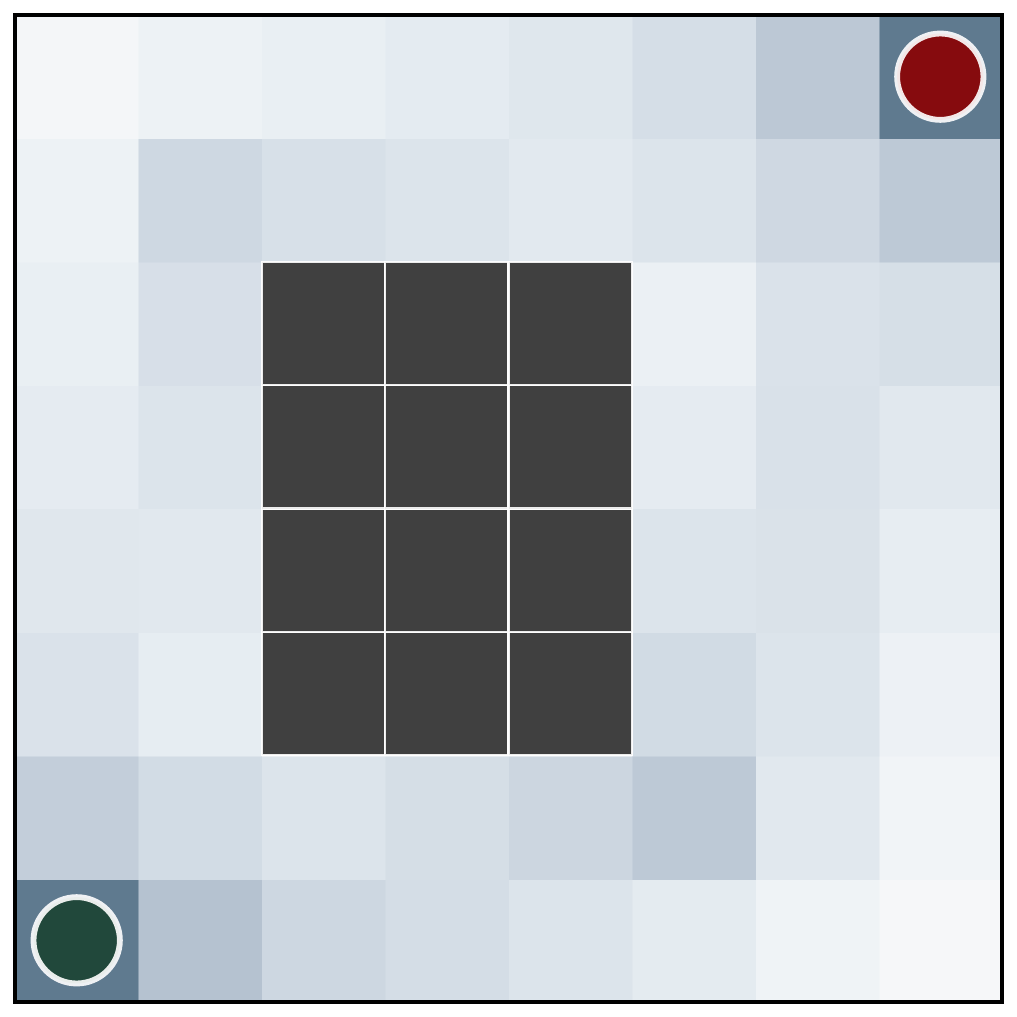}} 
    \caption{\textbf{Effect of tie-breaking strategies on trajectory distribution.} 
    The green and red dots mark the start and goal of the individual path-finding problem. 
    Colors denote the frequency of vertex visits across 2000 Monte-Carlo trials. 
    While both strategies yield individually optimal trajectories, random tie-breaking produces a skewed distribution of visits (std~$=387.1$) compared to the more uniform distribution under balanced tie-breaking and~$A^*$ (std~$=298.0$).}
    \label{fig:uniform_tie_breaker}
\end{figure}

\begin{algorithm}[tb]
    \resizebox{\columnwidth}{!}{
    \begin{minipage}{1\columnwidth}
    \small
    \caption{Individual optimal planning}
    \label{alg:individual-optimal-planning}
    \begin{algorithmic}[1] 
    
    \Require{Instance $\cI_t = \tup{G^t,A^t,\rho_s^t, \rho_g^t}$, current state $x_t$, horizon $H$}
    \Ensure{Individual optimal paths for the next $H$ steps $\hat{\vtau}_{\mathrm{plan}}^{A^t}$}

    \State $\hat{\vtau}_{\mathrm{plan}}^{A^t} \leftarrow \{\}$
    \ForAll{$a_i \in A^t$}
        \State $v_{\mathrm{current}}^{a_i} \leftarrow x_t(a_i)$
        \State $\hat\tau_{\mathrm{plan}}^{a_i} \leftarrow [v_{\mathrm{current}}^{a_i}]$
        \ForAll{$t' \leftarrow 1 \ldots H$}
            \LComment{Individual planning leveraging \cref{eq:candidates} and \cref{lem:optimal-path-distribution}}
            \State $v_{\mathrm{next}}^{a_i} \leftarrow $\Call{BalancedOptimalPlanning}{$\cI_t, v_{\mathrm{current}}^{a_i}, a_i$}
            \State $\hat\tau_{\mathrm{plan}}^{a_i}.\textbf{push}(v_{\mathrm{next}}^{a_i})$
            \State $v_{\mathrm{current}}^{a_i} \leftarrow v_{\mathrm{next}}^{a_i}$
        \EndFor
        \State $\hat{\vtau}_{\mathrm{plan}}^{A^t} \leftarrow \hat{\vtau}_{\mathrm{plan}}^{A^t} \cup \{\hat\tau_{\mathrm{plan}}^{a_i}\}$
    \EndFor
    \State \Return{$\hat\vtau_{\mathrm{plan}}^{A^t}$} 

    \Statex

    \Function{BalancedOptimalPlanning}{$\cI_t, v, a$}
    \LComment{Get candidate neighboring vertices as per~\cref{eq:candidates}}
    \State $\mathrm{cand}(a) \leftarrow \Call{GetCandidateNeighbors}{\cI_t,x_t,a}$
    \LComment{Choose the next vertex as per~\cref{eq:balance-tie-breaker} in~\cref{lem:optimal-path-distribution}}
    \State $v_{\mathrm{next}} \leftarrow \Call{WeightedSampling}{\cI_t,x_t,a}$
    \State \Return $v_{\mathrm{next}}$
    \EndFunction
    
    \end{algorithmic}
    \end{minipage}
}
\end{algorithm}

\myparagraph{\circled{2} Conflict detection} --
Once individual trajectories are generated in \circled{1}, the next step is to determine which agents can proceed without further coordination and which require conflict resolution (Algorithm \ref{alg:conflict-detection}).
A \emph{conflict} occurs when two agents occupy the same vertex at the same timestep (vertex conflict) or attempt to traverse the same edge in opposite directions at the same time (edge conflict).
Detecting such conflicts over the~$H$-step horizon allows us to \emph{factorize} the agent set for this timestep into i) conflict-free agents, denoted as $A_{\mathrm{CF}}^t$ whose trajectories can be finalized immediately, and ii) conflicting agents, denoted as $A_{\mathrm{C}}^t$, whose plans will be refined in later stages.

\begin{remark}[Effectiveness of the first-level factorization] \label{rem:effectiveness-first-level}
The first-level factorization removes conflict-free agents from the coordinated (re-)planning process, since they can safely follow their individual trajectories. 
As shown in \cref{tab:agent-number-reduction}, this step leads to a substantial reduction in the number of agents involved. 
In most cases, the problem size decreases dramatically, and even in extremely dense scenarios a meaningful reduction is achieved. 
These results also corroborate the presence of \emph{localized congestion}, as discussed in \cref{sec:factorization-compatible-traj}.
\end{remark}
\begin{table}[tb]
    \centering
    \caption{Reduction in number of agents after applying the first-level factorization (\anothercircled{2}) across maps and agent densities.}
    \label{tab:agent-number-reduction}
    
    \begin{minipage}[t]{0.47\linewidth}
        \centering
        \parbox[c][1.2em][c]{\textwidth}{\centering\textbf{\href{https://movingai.com/benchmarks/mapf/random-64-64-10.png}{\texttt{Random Map}}}}
        \vspace{0.5em}
        \begin{tabular}{@{}cc@{}}
            \toprule
            \textbf{\ $N$} & \textbf{${|A_{\mathrm{CF}}^t|}/{N}$} \\ [0.3em]
            \textbf{} & \textbf{(Agent reduction) [\%]} \\
            \midrule
            100 & 92.03  \\
            400 & 67.96  \\
            700 & 46.74  \\
            1000 & 27.44 \\
            1300 & 13.45 \\
            1600 & 8.79 \\
            \bottomrule
        \end{tabular}
    \end{minipage}
    \begin{minipage}[t]{0.47\linewidth}
        \centering
        \parbox[c][1.2em][c]{\textwidth}{\centering\textbf{\href{https://movingai.com/benchmarks/mapf/warehouse-20-40-10-2-2.png}{\texttt{Warehouse Map B}}}}
        \vspace{0.5em}
        \begin{tabular}{@{}cc@{}}
            \toprule
            \textbf{\ $N$} & \textbf{${|A_{\mathrm{CF}}^t|}/{N}$} \\ [0.3em]
            \textbf{} & \textbf{(Agent reduction) [\%]} \\
            \midrule
            1000 & 90.43  \\
            2000 & 80.82  \\
            3000 & 68.83  \\
            4000 & 53.87 \\
            5000 & 35.04 \\
            6000 & 21.72 \\
            \bottomrule
        \end{tabular}
    \end{minipage}
\end{table}

A naive all-pairs check requires~$\mathcal{O}(H\cdot N^2)$ operations, where~$N=\vert A^t\vert$, which quickly becomes prohibitive for large teams.
To overcome this, we introduce a \emph{spatial hashing} method that efficiently identifies agents in conflict and thus requiring replanning over the next~$H$ steps (see~\cref{alg:conflict-detection}).
The key idea is to maintain a hash table that maps each space-time cell~$\tup{v,\traj}$ to the (at most one) agent occupying it in its individual plan.
This reduces conflict detection to~$\mathcal{O}(H\cdot N)$ time, since each of the~$H$ planned positions per agent is inserted and checked exactly once.
This hashing framework naturally handles vertex conflicts; similarly, edge conflicts can be detected by hashing directed edges~$\tup{v, v', \traj}$ instead of vertices.
In either case, the method ensures that the first-level factorization starts from the complete set of true conflicts while avoiding the quadratic cost of naive pairwise checks.

\begin{algorithm}[tb]
    \resizebox{\columnwidth}{!}{
    \begin{minipage}{1\columnwidth}
        \small
        \caption{Finite-horizon conflict detection}
        \label{alg:conflict-detection}
        \begin{algorithmic}[1]
        
            \Require Current instance $\cI_t$, individual trajectories for the next $H$ steps $\hat{\vtau}_{\text{plan}}^{A^t}$, horizon $H$, agent labels $\calP_t$
            \Ensure{Conflict-free agents $A_{\mathrm{CF}}$, conflicting agents $A_{\mathrm{C}}$}
            \State $\mathrm{conflicted} \gets$ $\{\tup{a_i, \mathsf{False}} \mid a_i \in A^t\}$
            \State $\mathcal{M} \gets \Call{BuildSpatialHash}{\hat{\vtau}_{\text{plan}}^{A^t}, H}$
            \LComment{\textbf{Vertex conflicts checking} (\cref{def:path-conflicts}) \\ Mark $\mathrm{conflicted}(a_i) = \mathsf{True}$ if $a_i$ has vertex conflicts}
            \State $\mathrm{conflicted} \leftarrow \Call{CheckVertexConflicts}{\mathcal{M}, H}$
            \LComment{\textbf{Edge conflicts checking} (\cref{def:path-conflicts}) \\ Mark $\mathrm{conflicted}(a_i) = \mathsf{True}$ if $a_i$ has edge conflicts}
            \State $\mathrm{conflicted} \leftarrow \Call{CheckEdgeConflicts}{\mathcal{M}, H}$
            \State $A_{\mathrm{CF}} \leftarrow \{a_i \mid \mathrm{conflicted}(a_i) = \mathsf{False}\}$
            \State $A_{\mathrm{C}} \leftarrow \{a_i \mid \mathrm{conflicted}(a_i) = \mathsf{True}\}$
            \State \Return $A_{\mathrm{CF}}, A_{\mathrm{C}}$
        
            \Statex

            \Function{BuildSpatialHash}{$\hat{\vtau}_{\mathrm{plan}}^{A^t}, H$}
                \State $\mathcal{M} \gets \{\}$
                \ForAll{$t' \gets 0 \ldots H$}
                    \State $\mathcal{M} \gets \mathcal{M} \cup \{(t', \{\})\}$
                    \ForAll{$a_i \in A^t$}
                        \State $v \gets \hat{\tau}_{\mathrm{plan}}^{a_i}[t']$
                        \If{$v\in \mathrm{keys}(\mathcal{M}(t'))$}
                            \State $\mathcal{M}(t')(v) \gets \mathcal{M}(t')(v) \cup \{a_i\}$
                        \Else
                            \State $\mathcal{M}(t') \leftarrow \mathcal{M}(t') \cup \{\tup{v,a_i}\}$
                        \EndIf
                    \EndFor
                \EndFor
                \State \Return{$\mathcal{M}$}
            \EndFunction
    
        \end{algorithmic}
    \end{minipage}
}
\end{algorithm}

\myparagraph{\circled{3} Reachability grouping} --
After the first-level factorization separates conflict-free agents from those with conflicts, {we finalize the trajectories of the conflict-free agents without modification, while} the set of conflicting agents~$A_{\mathrm{C}}^t$ is \emph{factorized a second time} into smaller, independent groups (Algorithm \ref{alg:reachability-grouping}).
Two agents are placed in the same group if, within the next~$H$ timesteps, their reachable regions (conditioned on the fixed, finalized trajectories for conflict-free agents) intersect in space-time.
This ensures that agents in different groups cannot influence each other's feasibility during subsequent planning.

\begin{definition}[Reachable region]
\label{def:reach-region}
For an agent~$a_i\in A_{\mathrm{C}}^t$ at current vertex~$v_t^i$, the \emph{reachable region}~$\reachable_H(a_i)\subseteq V^t \times \{1,\ldots,H\}$ is the set of space-time vertices~$\tup{v,\traj}$ that the agent can occupy within at most~$H$ steps, following any path on $G_t$ that avoids collisions with finalized trajectories of conflict-free agents. 
\end{definition}

Operationally,~$\reachable_H(a_i)$ is computed via a BFS from~$v_t^i$ truncated at depth~$H$, with branching factor~$b$ and pruning if the path collides with conflict-free agents.
A spatial hashing table maps each space-time vertex to the (at most one) agent that first visits it during BFS exploration.
Whenever two agents' reachable regions intersect in the table, their groups are merged leveraging disjoint set union (DSU) structures~\cite{cormen2022introduction}. 

\begin{algorithm}[tb]
    \resizebox{\columnwidth}{!}{
    \begin{minipage}{1\columnwidth}
    \small
    \caption{Multi-Agent Reachability Grouping}
    \label{alg:reachability-grouping}
    \begin{algorithmic}[1]

    \Require{Current instance $\cI_t$, conflicting agents $A_{\mathrm{C}}$, current state $x_t$, horizon $H$}
    \Ensure{Reachability groups $\mathcal{G} = \{G_1, G_2, \ldots, G_k \mid G_i \subseteq A_{\mathrm{C}}\}$}
    

        \LComment{Compute $H$-step reachable sets and build mapping}
        \State $\mathcal{R} \leftarrow \Call{ComputeReachableSets}{\cI_t, A_{\mathrm{C}}, x_t, H}$
        \State $M \leftarrow \{\}$ 

        \ForAll{$a_i \in A_{\mathrm{C}}$}
            \ForAll{$v \in \mathcal{R}[a_i]$}
                \If{$v \in \mathrm{keys}(M)$}
                    \State $M(v) \gets M(v) \cup \{a_i\}$
                \Else
                    \State $M \gets M \cup \{(v, \{a_i\})\}$
                \EndIf
            \EndFor
        \EndFor

        \LComment{Union agents with overlapping reachable sets}
        \ForAll{$v \in \mathrm{keys}(M)$}
            \State $S \leftarrow M(v)$
            \If{$|S| > 1$}
                \State $a_{\text{rep}} \leftarrow \Call{Representative}{S}$
                \ForAll{$a \in S \setminus \{a_{\text{rep}}\}$}
                    \State $\Call{Union}{a_{\text{rep}}, a}$
                \EndFor
            \EndIf
        \EndFor

        \LComment{Extract connected components}
        \State $\mathcal{C} \leftarrow \{\}$
        \ForAll{$a_i \in A_{\mathrm{C}}$}
            \State $r \leftarrow \Call{Find}{a_i}$
            \If{$r \in \mathrm{keys}(\mathcal{C})$}
                \State $\mathcal{C}(r) \leftarrow \mathcal{C}(r) \cup \{a_i\}$
            \Else
                \State $\mathcal{C} \gets \mathcal{C} \cup \{(r, \{a_i\})\}$
            \EndIf
        \EndFor
        \State \Return{$\{\mathcal{C}(r) \mid r \in \mathrm{keys}(\mathcal{C})\}$}
    \Statex
    \Function{ComputeReachableSets}{$\cI_t, A_{\mathrm{C}}, x_t, H$}
        \State $\mathcal{R} \leftarrow \{\}$
        \ForAll{$a_i \in A_{\mathrm{C}}$}
            \State $\mathcal{R} \gets \mathcal{R} \cup \{(a_i, \emptyset)\}$
            \State $Q \leftarrow \langle x_t(a_i) \rangle$, $V \leftarrow \{x_t(a_i)\}$, $d[x_t(a_i)] \leftarrow 0$
        \EndFor
        \While{$Q \neq \emptyset$}
            \State $v \leftarrow Q.\text{dequeue}()$
            \State $\mathcal{R}(a_i) \gets \mathcal{R}(a_i) \cup \{v\}$
            
            \If{$d[v] < H$}
                \ForAll{$u \in \text{Neighbors}(v)$}
                    \If{$u \notin V$}
                        \State $Q.\text{enqueue}(u)$
                        \State $V \leftarrow V \cup \{u\}$
                        \State $d[u] \leftarrow d[v] + 1$
                    \EndIf
                \EndFor
            \EndIf
        \EndWhile
        \State \Return{$\mathcal{R}$}
    \EndFunction
    \end{algorithmic}
    \end{minipage}
    }
\end{algorithm}

\begin{lemma}[Correctness of reachability-based grouping]
\label{lem:reach}
Let~$\cG_1,\ldots,\cG_m$ be the groups produced by the BFS and DSU reachability grouping procedure.
If~$a_i\in \cG_p$ and~$a_j\in \cG_q$ with~$p\neq q$, then for all~$\tup{v,\traj}\in \reachable_H(a_i)$ and~$\tup{v',\traj'}\in \reachable_H(a_j)$, we have~$\tup{v,\traj}\neq \tup{v',\traj'}$, i.e., no space-time conflicts are possible between groups in the next~$H$ steps. 
In particular, no vertex and edge conflicts are possible between the agents in different groups in the next $H$ steps. 

\end{lemma}

\begin{proof}
Agents are merged into the same DSU set exactly when their reachable regions overlap at some space-time vertex.
Thus, if~$a_i$ and~$a_j$ remain in different DSU sets, their reachable regions must be disjoint. 
This immediately precludes vertex conflicts.
We prove the absence of edge conflicts by contradiction. Assume there is an edge conflict between~$a_i$ and~$a_j$. By definition (cf.~\cref{def:path-conflicts}), there must exist a time $\tau < H$ and two adjacent vertices $v,v' \in G^t$, such that $\tup{v, \tau}, \tup{v', \tau+1}\in \reachable_H(a_i)$ and $\tup{v', \tau}, \tup{v, \tau+1}\in \reachable_H(a_j)$. 
Since the regions are disjoint, we have $\tup{v, \tau+1}\notin \reachable_H(a_i)$. However, this can only happen when there is a conflict-free agent $a_k^{\mathrm{CF}}$ that occupies vertex $v$ at $\tau+1$, but that contradicts with $\tup{v, \tau+1}\in \reachable_H(a_j)$. Therefore, $a_i$ and $a_j$ will not have edge conflicts. 
\end{proof}

\begin{lemma}[Complexity] \label{lem:grouping-complexity}
Let~$N_{\mathrm{C}}=\vert A_{\mathrm{C}}^t\vert$ be the number of conflicting agents,~$b$ the branching factor of~$G^t$, and~$\alpha(\cdot)$ the inverse Ackermann function (cf. \cite{cormen2022introduction}).
The BFS and DSU grouping procedure runs in~$\mathcal{O}(N_{\mathrm{C}}\cdot b^H\cdot \alpha(N_{\mathrm{C}}))$ times and uses~$\mathcal{O}(N_{\mathrm{C}}\cdot b^H)$ space.
\end{lemma}

\begin{proof}
Each agent's BFS explores at most~$\mathcal{O}(b^H)$ space-time vertices, each inserted and queried in~$\mathcal{O}(1)$ expected time.
Each DSU operation (find/union) is~$\mathcal{O}(\alpha(N_{\mathrm{C}}))$ amortized (cf. Theorem 21.14 in \cite{cormen2022introduction}).
The total cost is therefore~$\mathcal{O}(N_{\mathrm{C}}\cdot b^H\cdot \alpha(N_{\mathrm{C}}))$.
\end{proof}

\begin{figure}[tb]
    \centering
    \includegraphics[width=\linewidth, trim=19cm 11.8cm 19cm 11.8cm, clip]{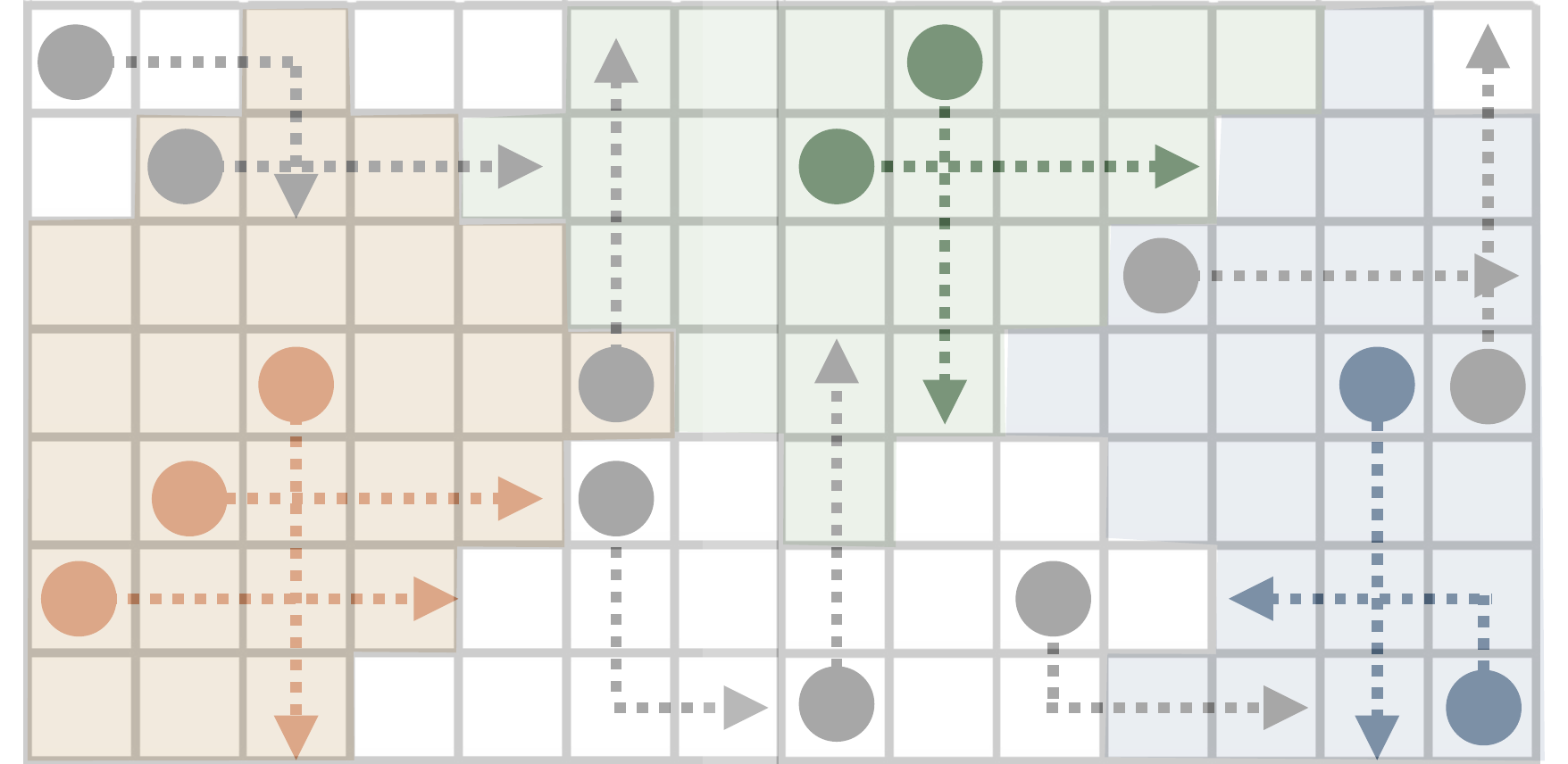}
    \caption{Reachability-based grouping with a horizon of~$H=3$. 
Gray agents are conflict-free and thus exempt from replanning. 
Colored agents, which will encounter conflicts within the next~$H$ steps, are partitioned into groups (distinct solid colors). 
Each group is defined by overlapping reachable regions, depicted as translucent areas matching the agent colors.}
    \label{fig:reachability_grouping}
\end{figure}

\begin{example}[Example of reachability-based grouping]
\label{ex:reachability-grouping}
In \cref{fig:reachability_grouping}, the gray agents are conflict-free, meaning their individual trajectories (gray dotted lines) are finalized for the next~$H$ steps, where $H=3$ in this example.
Agents in other colors are the conflicting agents.
Here, the orange agents, the green agents, and the blue agents form separate reachability groups respectively: within each group, reachable regions (marked as translucent colored regions) intersect in space-time, but no such intersection exists between any two groups.
\end{example}

\begin{remark}[Connection to classical reachability analysis]
The grouping criterion described above is an instance of a finite-horizon reachable set computation, a central concept in control theory~\cite{blanchini2008set, bertsekas2012dynamic}.
In the standard setting of discrete-time dynamical systems~$x_{k+1}=f(x_k,u_k)$, with~$x_k\in X,u_k\in U$, the~$H$-step reachable set from an initial condition~$x_0$ is
\begin{equation*}
    \reachable_H(x_0)=\{x_\tau\mid \exists \{u_i\}_{i=0}^{\tau-1} \text{ s.t. }x_\tau\text{ is reached in }\tau\leq H \text{ steps}\}.
\end{equation*}
Our~$\reachable_H(a_i)$ corresponds to this concept in the special case of a graph-based dynamics model, with control inputs representing allowed moves to neighboring vertices and additional state constraints induced by the frozen trajectories of finalized agents.
The disjointness property of the groups in \cref{lem:reach} is analogous to the separability of reachable sets in control: if two reachable sets are disjoint, their evolutions remain independent over the considered horizon.
\end{remark}

\myparagraph{\circled{4} Parallel factorization-compatible planning} --
In this stage, each independent group of conflicting agents produced by step \circled{3} is replanned in parallel over the next~$H$ steps, {(recall Definition \ref{def:cond-traj-single} and \ref{def:cond-traj-set})} with the finalized trajectories of the conflict-free agents from \circled{2}.
The compatibility is enforced by maintaining a hash map of space-time cells occupied by conflict-free agents, allowing constant-time checks that prevent replanned agents from entering those cells.
Because the non-reachability property of \circled{3} ensures that different groups have disjoint reachable zones, each group can be planned independently without introducing inter-group conflicts.
For each group, we employ a modified version of the PIBT algorithm~\cite{okumura2025lightweight}.
The modification incorporates hindrance information into PIBT's neighbor-selection rule, biasing agents away from moves that could lead to future deadlocks with fixed obstacles or other finalized trajectories. Besides that, the balanced tie-breaker described in~\cref{lem:optimal-path-distribution} is applied as well to uniformly distribute the trajectories and prevent congestion.

\begin{lemma}[Compatibility]
\label{lem:two-comp}
Let~$\cG_i$ and $\cG_j$ be any two groups of conflicting agents produced by the BFS and DSU reachability grouping procedure in \circled{3} with $i \neq j$. Suppose the resulting trajectories of them after replanning are $\hat\setoftraj^{\cG_i}$ and $\hat\setoftraj^{\cG_j}$. 
Then $\hat\setoftraj^{\cG_i} \mid \hat\setoftraj^{\cG_j}$. 
\end{lemma}
\begin{proof}
According to \cref{lem:reach}, $\forall a_i \in \cG_i$ and $\forall a_j \in \cG_j$, $\traj^{a_i}$ and $\traj^{a_j}$ are conflict-free over the next $H$ timesteps. Then by the definitions of factorization-compatible trajectories and set (\cref{def:cond-traj-single}, \cref{def:cond-traj-set}), $\hat\setoftraj^{\cG_i} \mid \hat\setoftraj^{\cG_j}$ is proved. 
\end{proof}

\begin{lemma}[Bounded total computation] \label{lem:replanning-complexity}
Let~$\mathcal{G}_k$ be a group of conflicting agents with reachable region~$\reachable_k$ over horizon~$H$, and let~$G_{\mathrm{cond}}^k$ be the subgraph of~$\reachable_k$ obtained by removing all space-time cells occupied by conflict-free agents. The conflict resolution is guaranteed to be finished in $\vert \cG_k \vert$ iterations, and the time complexity is bounded by $\mathcal{O}(\vert \cG_k \vert\log\vert \cG_k \vert)$. 
\end{lemma}

\begin{proof}
The time complexity of PIBT follows from~\cite{okumura2022priority}. In our case, although the introduction of factorizable-compatible agents (conflict-free agents) will further shrink the usable space, PIBT's mechanism operates unchanged on $G_{\mathrm{cond}}^k$. Agents will be sorted by their initial priorities, and each agent will be visited sequentially in order of priorities with inheritance, and determine its next move among the neighbors and the current vertex. 
\end{proof}

\begin{remark}[Anytime and modular factorization-compatible planning]
The modular design of \gls{acr:fico}, illustrated in \cref{fig:fico_algo}, enables seamless integration with a broad class of \gls{acr:mapf} solvers. 
In particular, the factorization-compatible planner in step~\circled{4} of \cref{alg:fico} can be instantiated with different algorithms, thereby turning static methods into adaptive frameworks that incorporate feedback naturally while preserving the computational and structural advantages of factorization. 
For example, \gls{acr:fico}+PIBT corresponds to the baseline version in \cref{sec:fico-algo-summary}, while \gls{acr:fico}+CBS~\cite{sharon2015conflict} or EECBS~\cite{li2021eecbs} provide optimality guarantees within the factorization window, and LaCAM*~\cite{okumura2023improving} or Anytime PIBT~\cite{gandotra2025anytime} yield anytime improvements, with PIBT itself serving as a fast anytime backup. 
A systematic evaluation of these variants is beyond the scope of this paper and left for future work.
\end{remark}

\myparagraph{\circled{5} Congestion resolution} --
Factorization-compatible planning described in \circled{4} aims at resolving the conflicts in all conflicting agents in $A_{\mathrm{C}}^t$ by replanning them group by group in parallel. In certain cases, this process for a group~$\mathcal{G}_k$ may fail because the finalized trajectories of nearby conflict-free agents fully occupy the reachable region~$\reachable_k$ over~$H$.
In such cases, we make the algorithm robust by expanding the conflicting agent set~$A_{\mathrm{C}}^t$ by merging it with the minimal set of nearby conflict-free agents whose trajectory's removal from the finalized trajectories yields a non-empty feasible region and allows for successful replanning for all conflicting agents. 
Formally, let $\reachable_H^{A_{\mathrm{C}}^t} = \bigcup_{a\in A_{\mathrm{C}}^t} \reachable_H(a)$  be the entire reachable region for all conflicting agents. 
Define the blocking set
\begin{equation*}
    A_B^t=\{a\in A_{\mathrm{CF}}^t\mid \traj^a\cap \reachable_H^{A_{\mathrm{C}}^t}\neq \emptyset\},
\end{equation*}
where~$\traj^a$ is the planned~$H$-step trajectory of a conflict-free agent~$a$ in step~\circled{1}.
The failure in factorization-compatible replanning implies that the usable space for replanning, $\reachable_H^{A_{\mathrm{C}}^t} \setminus \bigcup_{a\in A_B^t}\traj^a$, is not enough, which results in no feasible plan for conflicting agents, represented as $\mathcal{P}\left(\reachable_H^{A_{\mathrm{C}}^t} \setminus \bigcup_{a\in A_B^t}\traj^a\right)=\emptyset$. 
Here, we define the enlargement operator
\begin{equation*}
    \mathcal{E}(A_{\mathrm{C}}^t)=A_{\mathrm{C}}^t\cup \argmin_{S\subseteq A_B^t, \mathcal{P}\left(\reachable_H^{A_{\mathrm{C}}^t} \setminus \bigcup_{a\in S}\traj^a\right)\neq\emptyset}\vert S\vert.
\end{equation*}
However, $\argmin S$ cannot be directly computed. In practice, an iterative estimation is applied to estimate the enlarged $A_{\mathrm{C}}^t$
\begin{equation*}
    \hat\cE(A_{\mathrm{C}}^t)=A_{\mathrm{C}}^t\cup \hat S
\end{equation*}
\begin{equation*}
 \textup{where}~~   \hat S = \left\{a\in A_{\mathrm{CF}}^t \mid a \text{ is among the top } d \text{ closest to }A_{\mathrm{C}}^t\right\} \footnote{The ``among the top $d$ closest to $A_{\mathrm{C}}^t$'' requirement could be formally written as $\hat S = \left\{a\in A_{\mathrm{CF}}^t \mid \left\vert \left\{ b \in A_{\mathrm{CF}}^t \mid \distance(b, A_{\mathrm{C}}^t) < \distance(a, A_{\mathrm{C}}^t) \right\} \right\vert \leq d \right\}$, where $\distance(a, A_{\mathrm{C}}^t) = \min_{a'\in A_{\mathrm{C}}^t} \distance(a,a')$ represents the closest distance to the conflicting agent set. }
\end{equation*}
and $d$ is a tunable parameter. This enlargement procedure then updates $A_{\mathrm{C}}^t \leftarrow \hat \cE(A_{\mathrm{C}}^t)$, $A_{\mathrm{CF}}^t = A^t \setminus A_{\mathrm{C}}^t$, and repeats until a feasible $H$-step plan exists. 
If all agents are eventually merged into a single group (worst case), the process degenerates to classical PIBT, which is guaranteed to return a valid next step when one exists.

\begin{lemma}[Termination and worst-case] \label{lem:termination-to-pibt}
Starting from~$A_{\mathrm{C}}^t$, repeatedly applying the recursion~$A_{\mathrm{C}}^t \leftarrow \hat\cE(A_{\mathrm{C}}^t)$ terminates after at most~$\vert A_{\mathrm{CF}}^t\vert$ iterations.
In the worst case,~$A_{\mathrm{C}}^t=A^t$ and the method reduces to PIBT, which is complete for finite-horizon \gls{acr:mapf}.
\end{lemma}

\begin{proof}
Suppose Each iteration strictly enlarges~$\vert A_{\mathrm{C}}^t\vert$ by at least $\min(d,\vert A^t_{\mathrm{CF}} \vert)$ agents from~$A_{\mathrm{CF}}^t$,
Since~$\vert A_{\mathrm{CF}}^t \vert$ is finite, the process stops after at most~$\vert A_{\mathrm{CF}}^t \vert$ iterations and then all agents will be replanned freely using PIBT. 
\end{proof}

\begin{remark}[Difference from global PIBT] \label{rem:difference-with-global-pibt}
In the worst case, \gls{acr:fico} falls back to PIBT.
The only distinction is that, in \gls{acr:fico}, agents are first partitioned into reachability-based groups (step~\circled{3}).
Each group is then planned independently with PIBT.
Because the groups are disjoint over the~$H$-step horizon, this decomposition does not alter the solution: the outcome is identical to running a ``global'' PIBT on all agents together.
The advantage is computational: planning can be carried out in parallel across groups.
\end{remark}

\begin{remark}[Connections to control]
This enlargement procedure can be viewed as a discrete analog of constraint relaxation in \gls{acr:mpc}: when the constrained feasible set becomes empty, the algorithm selectively relaxes constraints (finalized trajectories) until the set becomes non-empty, preserving feasibility while minimizing deviation from original conditioning.
\end{remark}

\myparagraph{\circled{6} Merging} -- 
Once all groups have been successfully replanned, their resulting~$H$-step trajectories are combined into a single joint plan for the current timestep.
Following the receding-horizon paradigm, we do not execute the entire~$H$-step plan.
Instead, we extract only the first step of each agent, and treat the remaining steps as a lookahead reference to be discarded after execution. The trajectories of all conflict-free agents are cached for reuse in future planning.

\begin{lemma}[Consistency of merged plans]
Let $\cG_1,\ldots,\cG_m$ be a partition of the conflicting agents $A_{\mathrm{C}}^t$ constructed using BFS and DSU in \circled{3} and are replanned in the conflict-resolution in \circled{4}, while the conflict-free agents $A_{\mathrm{CF}}^t$ keep their original individual trajectories, then the merged set of first step actions~$\{u_t^a\mid a\in A^t\}$ is globally conflict-free at timestep~$t+1$.
\end{lemma}
\begin{proof}
\cref{lem:two-comp} ensures the compatibility between any two distinct groups constructed and replanned in the conflict-resolution process in \circled{4},\circled{5}. By construction, each group is compatible with the conflict-free agent set. Then the compatibility of all agents' trajectories over the next $H$ steps could be guaranteed using~\cref{lem:comp-fact}. In particular, the conflict-freeness at step~1 is obtained. 
\end{proof}

\subsection{Theoretical analysis}

\begin{lemma}[Completeness] \label{lem:completeness}
Let~$x_t$ be the state of the problem at time~$t$ (\cref{def:mapf-system}). 
Given~$x_t$ and the current instance~$\cI_t$ as input, if a valid next-step movement~$\hat u_t$ exists, then \gls{acr:fico} will find such a movement in finite time.
\end{lemma}

\begin{proof}
By \cref{lem:termination-to-pibt} and \cref{rem:difference-with-global-pibt}, in the worst case \gls{acr:fico} reduces to PIBT with reachability-based factorization. 
It is known that PIBT is single-step complete~\cite{okumura2022priority}: it processes agents sequentially according to their priorities and employs backtracking to ensure that all agents receive a valid move. 
Agents that cannot progress are backtracked and remain at their current location, while already-processed agents keep valid assignments. 
Thus, whenever a valid next-step movement exists, PIBT, and therefore \gls{acr:fico}, is guaranteed to find one in finite time.
\end{proof}

\begin{lemma}[Complexity] \label{lem:computational-complexity-normal}
Let~$N=\lvert A^t\rvert$ be the number of agents under consideration, $N_{\mathrm{C}}=\vert A_{\mathrm{C}}^t \vert$ the number of conflicting agents,~$H$ the lookahead horizon,~$b$ the branching factor, and~$N_{\mathrm{thread}}$ the number of parallel threads.
Assume that the distance heuristic and optimal-path-count queries used in \gls{acr:fico} have amortized constant cost (due to lazy caching; cf.\ \cref{alg:perfect_distance_heuristic,alg:optimal_path_count}). 
Then one replanning cycle of \gls{acr:fico} \emph{without} congestion resolution (steps $\circled{1}\!\to\!\circled{2}\!\to\!\circled{3}\!\to\!\circled{4}\!\to\!\circled{6}$) runs in
\[
\mathcal{O}\!\big(HN\big)\;+\;\mathcal{O}\!\big(N_{\mathrm{C}}\, b^{H}\,\alpha(N_{\mathrm{C}})\big)\;+\;
\mathcal{O}\!\Big(\big\lceil\tfrac{N_{\mathrm{C}}}{N_{\mathrm{thread}}}\big\rceil \log \big\lceil\tfrac{N_{\mathrm{C}}}{N_{\mathrm{thread}}}\big\rceil\Big)
\]
time {under a balanced-load asssumption in step~$\circled{4}$}.
In the worst case of complete load imbalance in step~$\circled{4}$, the last term becomes $\mathcal{O}(N_{\mathrm{C}}\log N_{\mathrm{C}})$.
\end{lemma}

\begin{proof}
The pipeline runs sequentially across major stages.

\emph{Step~$\circled{1}$ (independent $H$-step proposals).} 
Each agent selects $H$ next vertices with $O(1)$ amortized work per selection, in parallel, giving $O\!\big((HN)/N_{\mathrm{thread}}\big)$ wall-clock time.

\emph{Step~$\circled{2}$ (conflict detection).} 
Spatial hashing checks each of the $H$ planned positions per agent once, for $O(HN)$ total time (see \cref{alg:conflict-detection}). 
This term dominates step~$\circled{1}$ for $N_{\mathrm{thread}}\ge 1$, so together steps $\circled{1}$–$\circled{2}$ cost $O(HN)$.

\emph{Step~$\circled{3}$ (reachability grouping).} 
By \cref{lem:grouping-complexity}, grouping conflicting agents costs $O\!\big(N_{\mathrm{C}}\, b^{H}\,\alpha(N_{\mathrm{C}})\big)$, where $\alpha(\cdot)$ is the inverse Ackermann term from union–find operations.

\emph{Step~$\circled{4}$ (parallel replanning by group).} 
By \cref{lem:replanning-complexity}, replanning a group $\cG$ costs $O\!\big(\lvert\cG\rvert \log \lvert\cG\rvert\big)$. 
With $N_{\mathrm{thread}}$ threads, the wall-clock time is governed by the largest group. 
Under a balanced-load assumption (e.g., near-uniform agent distribution), the largest group has size $\lceil N_{\mathrm{C}}/N_{\mathrm{thread}}\rceil$, yielding 
$O\!\big(\lceil N_{\mathrm{C}}/N_{\mathrm{thread}}\rceil \log \lceil N_{\mathrm{C}}/N_{\mathrm{thread}}\rceil\big)$. 
In the worst case (all conflicts concentrated), this term is $O(N_{\mathrm{C}}\log N_{\mathrm{C}})$.

\emph{Step~$\circled{6}$ (merge first actions).} 
Scanning all agents once is $O(N)$, which is subsumed by $O(HN)$.
Summing the terms gives the stated bound.
\end{proof}

Since the first-level factorization substantially reduces the effective fleet size (cf.~\cref{rem:effectiveness-first-level,tab:agent-number-reduction}), the number of conflicting agents $N_{\mathrm{C}}$ is typically much smaller than $N$. 
As a result, the overall complexity of \gls{acr:fico} approaches the near-linear bound $\mathcal{O}(NH)$ in practice.

\subsection{Discussion}
\label{sec:discussion}

The performance advantage of \gls{acr:fico} stems primarily from its mechanism for \emph{priority determination}. 
Previous work~\cite{okumura2025lightweight} has demonstrated that priority assignment is critical to the effectiveness of algorithms such as PIBT. 
In our framework, priorities are dynamically allocated by selecting the set of conflict-free agents ($A_{\mathrm{CF}}$) and granting them the highest priority, thereby ensuring that they can follow their individually optimal trajectories without interference. 
This selection is closely tied to the horizon parameter~$H$, which governs the scope of lookahead and hence the identification of~$A_{\mathrm{CF}}$. 
Furthermore, the closed-loop nature of \gls{acr:fico} allows it to naturally adapt to changes in the environment: any deviations are immediately incorporated through feedback, enabling robust operation in dynamic settings. 
These performance benefits are validated experimentally and are evident in \cref{fig:soc_comparison_oneshot_mapf,fig:throughput_comparison_lifelong_mapf}.

In addition to performance, \gls{acr:fico} also exhibits significant computational advantages. 
By focusing on one-step planning rather than solving for the entire horizon at once, the algorithm avoids the exponential blow-up associated with long-horizon joint planning. 
At the same time, the dual-level factorization scheme exploits the compositional structure of the \gls{acr:mapf} problem: agents are partitioned into conflict-free and conflicting sets, and the remaining conflicts are further decomposed into smaller, independent subproblems with horizon~$H$. 
These subproblems are amenable to parallel computation, leading to substantial gains in scalability, confirmed in our experiments and illustrated in \cref{fig:ERT_comparison_oneshot_mapf}.

\section{Experiments}
\label{sec:experiments}

We evaluate \gls{acr:fico} across multiple settings of the unified \gls{acr:mapf} problem.
Experiments are carefully designed to answer four key questions:
a) How does \gls{acr:fico} compare with state-of-the-art open-loop and closed-loop algorithms in terms of response time and solution quality?
b) How does it perform in lifelong \gls{acr:mapf}, where throughput is the key metric?
c) What is the contribution of each algorithmic component (ablation study)?
d) How robust is \gls{acr:fico} under execution uncertainties such as stochastic delays and agent additions?

All experiments are conducted on a MacBook Pro 2023 with a 12-core CPU and 36 GB of RAM.
Parallel computation is implemented via multi-threading over 12 threads. 
The maps used in the following sections are from the MAPF benchmark~\cite{stern2019mapf}, including \texttt{Empty Map} (\texttt{empty-48-48}), \texttt{Random Map} (\texttt{Random-64-64-10}), \texttt{Warehouse Map A} (\texttt{warehouse-20-40-10-2-1}), and \texttt{Warehouse Map B} (\texttt{warehouse-20-40-10-2-2}).\footnote{https://movingai.com/benchmarks/mapf/index.html}

\subsection{Performance in one-shot and lifelong \gls{acr:mapf}}
\label{sec:experiments-no-uncertainties}

\begin{figure}[tb]
    \centering
    \subfloat[\href{https://movingai.com/benchmarks/mapf/random-64-64-10.png}{\texttt{Random Map}}]{\includegraphics[width=0.5\linewidth]{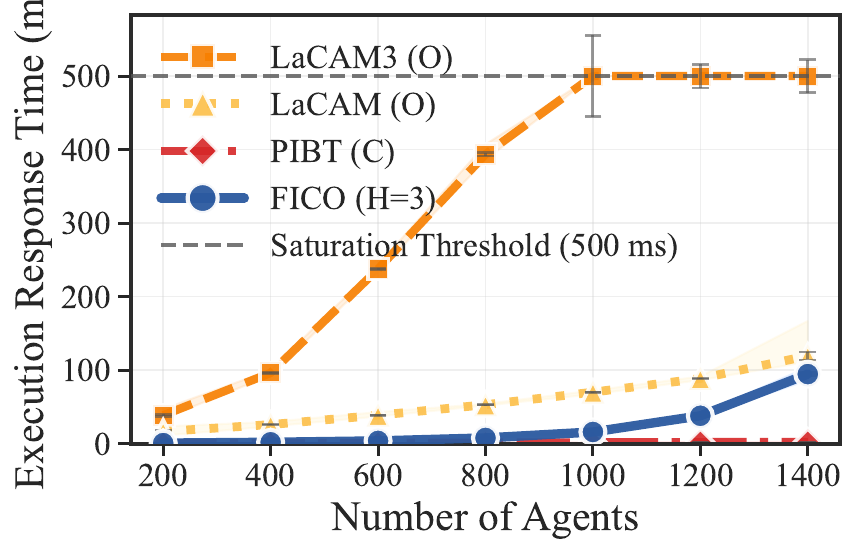}} \hfill
    \subfloat[\href{https://movingai.com/benchmarks/mapf/random-64-64-10.png}{\texttt{Random Map}}]{\includegraphics[width=0.5\linewidth]{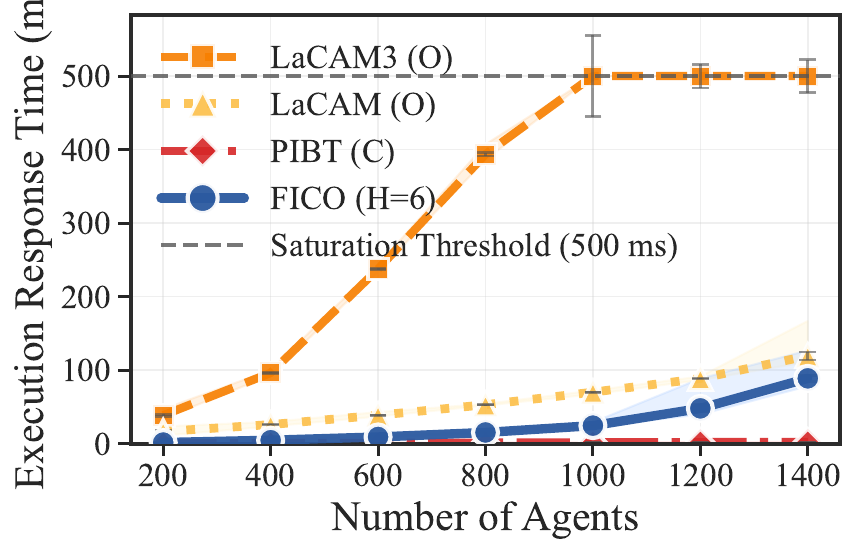}} \\
    \subfloat[\href{https://movingai.com/benchmarks/mapf/warehouse-20-40-10-2-2.png}{\texttt{Warehouse Map B}}]{\includegraphics[width=0.5\linewidth]{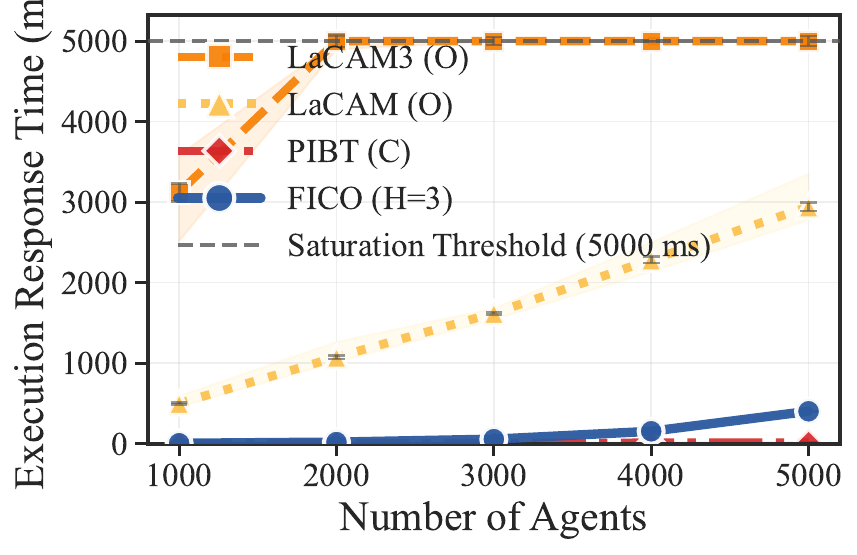}} \hfill
    \subfloat[\href{https://movingai.com/benchmarks/mapf/warehouse-20-40-10-2-2.png}{\texttt{Warehouse Map B}}]{\includegraphics[width=0.5\linewidth]{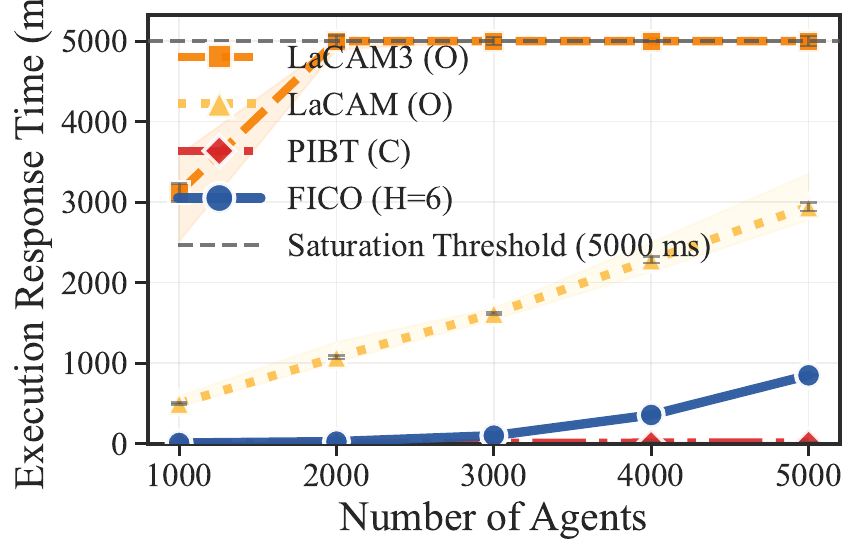}} 

    \caption{Execution Response Time (ERT) comparison. 
Although \gls{acr:fico} is slower than the closed-loop baseline PIBT, its~$H$-step lookahead design still enables real-time operation with near-instant responses, while achieving substantially faster execution than open-loop baselines.}
    \label{fig:ERT_comparison_oneshot_mapf}
    \vspace{-10pt}
\end{figure}

\begin{figure}[tb]
    \centering
    \subfloat[\href{https://movingai.com/benchmarks/mapf/empty-48-48.png}{\texttt{Empty Map}}]{\includegraphics[width=0.5\linewidth]{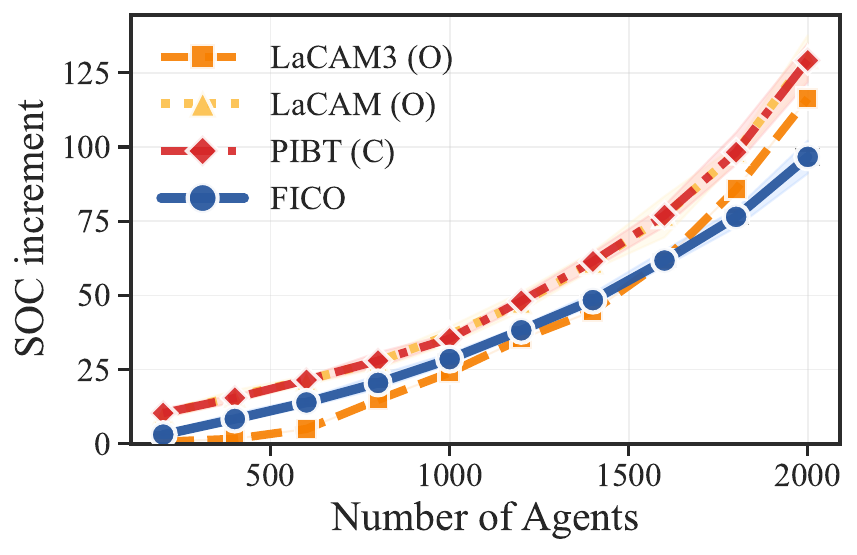}} \hfill
    \subfloat[\href{https://movingai.com/benchmarks/mapf/random-64-64-10.png}{\texttt{Random Map}}]{\includegraphics[width=0.5\linewidth]{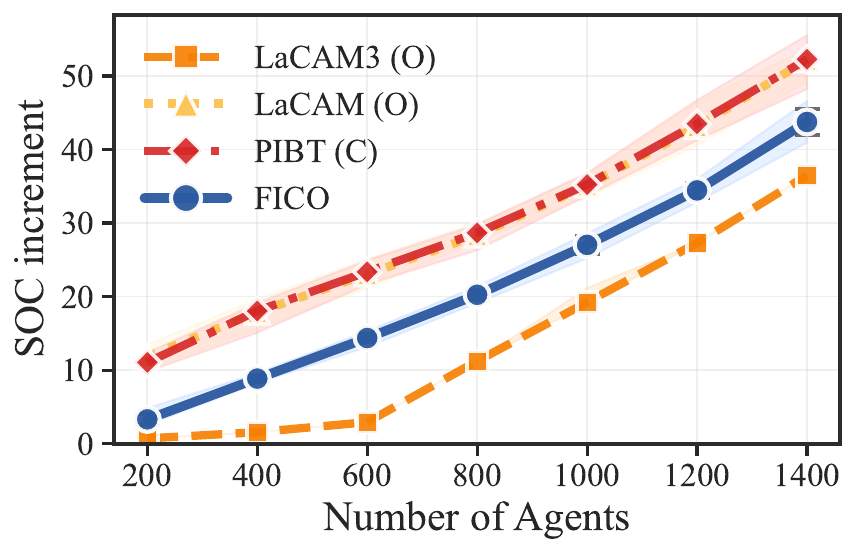}} \\
    \subfloat[\href{https://movingai.com/benchmarks/mapf/warehouse-20-40-10-2-1.png}{\texttt{Warehouse Map A}}]{\includegraphics[width=0.5\linewidth]{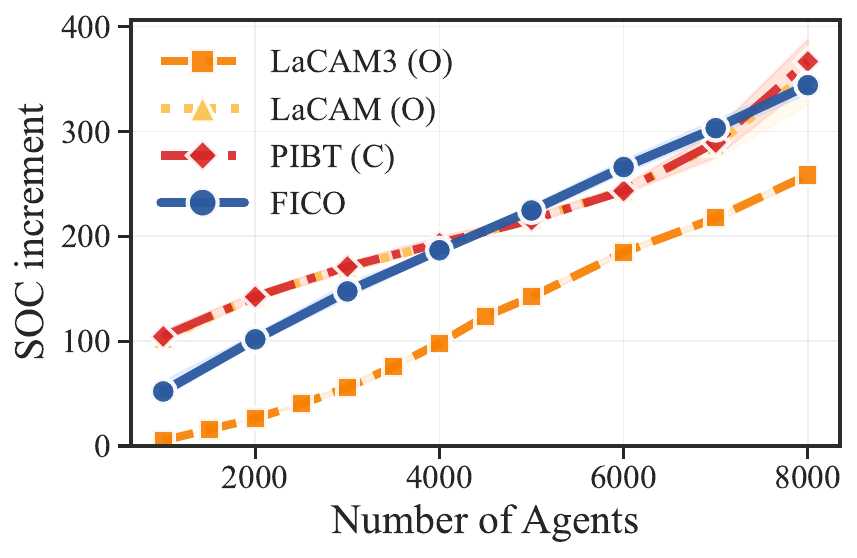}} \hfill
    \subfloat[\href{https://movingai.com/benchmarks/mapf/warehouse-20-40-10-2-2.png}{\texttt{Warehouse Map B}}]{\includegraphics[width=0.5\linewidth]{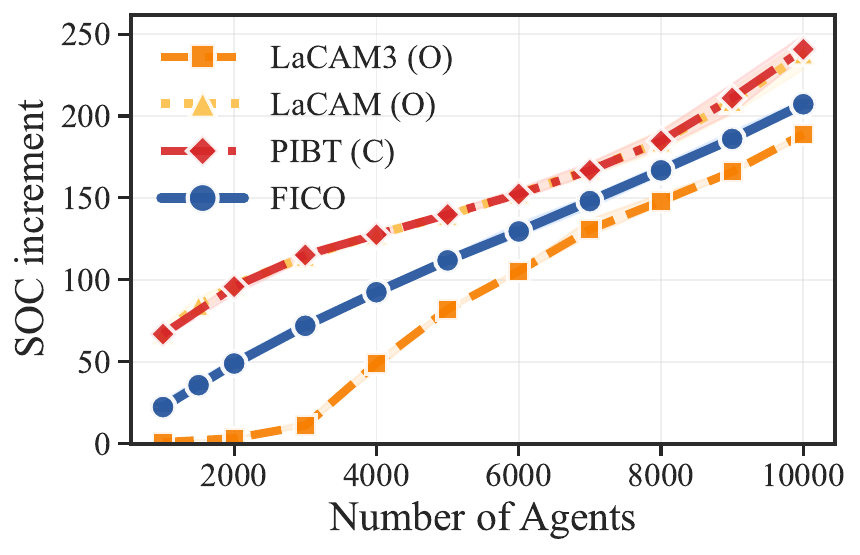}} 
    
    \caption{FICO consistently outperforms closed-loop baselines in terms of \gls{acr:soc} in one-shot \gls{acr:mapf} (\cref{prob:searching-one-shot-mapf}).}
    \label{fig:soc_comparison_oneshot_mapf}
    \vspace{-5pt}
\end{figure}

\myparagraph{One-shot \gls{acr:mapf}.}
We compare \gls{acr:fico} against PIBT~\cite{okumura2022priority} (closed-loop baseline, marked with (C)), Engineered LaCAM~\cite{okumura2023engineering} (open-loop baseline, marked as (O)), and LaCAM~\cite{okumura2023improving} (open-loop baseline, marked as (O)). 
We report (i) the \emph{\gls{acr:ert}}, the wall-clock time from planning start to when the first step is executable; for open-loop methods this equals total search time, whereas closed-loop methods can execute as soon as the first step is computed, and (ii) solution quality via the \gls{acr:soc} \emph{increment} ($\Delta\mathrm{SOC}$) over the theoretical lower bound.\footnote{We plot~$\Delta\mathrm{SOC} \!:=\! \mathrm{SOC}-\sum_{a\in A}\gamma(\rho_s(a),\rho_g(a))$, i.e., the coordination overhead beyond the sum of single-agent shortest-path lengths.}

As shown in \cref{fig:ERT_comparison_oneshot_mapf}, \gls{acr:ert} grows with the number of agents, yet \gls{acr:fico} achieves substantial speedups over the open-loop baseline, while remaining comparable in speed to the simpler PIBT due to parallelization and our algorithmic design. 
In the warehouse case study (\href{https://movingai.com/benchmarks/mapf/warehouse-20-40-10-2-2.png}{\texttt{Warehouse Map B}}), \gls{acr:fico} achieves quasi-instant responsiveness, initiating execution within 15\,ms even with 5000 agents, and reduces computation time by $77.1\%$ and $97.7\%$ compared with open-loop baselines.


For solution quality, we compare~$\Delta\mathrm{SOC}$ in \cref{fig:soc_comparison_oneshot_mapf}.
To ensure a fair trade-off, the anytime baselines continue refining until \gls{acr:fico} finishes all planning steps.
The results show that \gls{acr:fico} consistently delivers competitive solution quality relative to established baselines. 
Its finite-horizon lookahead enables higher-quality solutions than closed-loop PIBT variants, while avoiding the prohibitive computation times of open-loop solvers such as LaCAM3, which only occasionally achieve better \gls{acr:soc} performance.
Overall, \gls{acr:fico} lies on a favorable Pareto frontier, combining near-PIBT responsiveness with markedly better solution quality.

\myparagraph{Lifelong \gls{acr:mapf}.}
Additionally, we next evaluate throughput in lifelong \gls{acr:mapf}, a setting naturally captured by the unified model with a dynamic environment, as the case of changeable goals with fixed-horizon termination. 
Unlike methods that require precomputed guidance structures~\cite{zhang2024guidance}, which can introduce substantial delays, \gls{acr:fico} operates in a true \emph{zero-shot} manner, applying directly without modification, therefore enabling instant execution. Here we compare \gls{acr:fico} with the closed-loop baseline PIBT algorithm, which is also \emph{zero-shot} in the sense that no expensive pre-computation is required. 
As shown in \cref{fig:throughput_comparison_lifelong_mapf}, beyond its computational advantages, the dual-level factorization in \gls{acr:fico} consistently yields higher throughput than the closed-loop PIBT baseline across all maps and scales.
The gains are most pronounced in dense environments, where coordination is especially challenging. 
In the extreme case of 10{,}000 agents in the warehouse setting, \gls{acr:fico} achieves a $59.2\%$ throughput improvement over the baseline.
This highlights the effectiveness of \gls{acr:fico} for large-scale, continuous task execution.

\begin{figure}[tb]
    \centering
    \subfloat[\href{https://movingai.com/benchmarks/mapf/empty-48-48.png}{\texttt{Empty Map}}]{\includegraphics[width=0.5\linewidth]{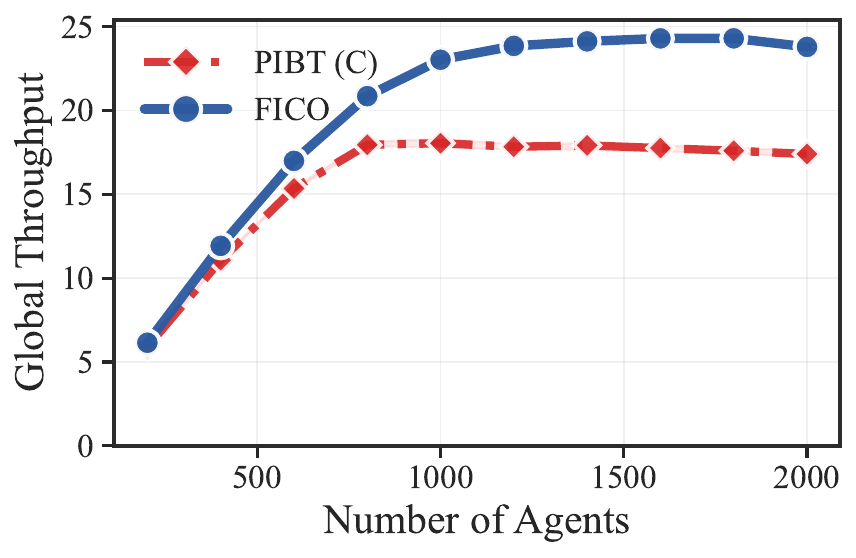}} 
    \subfloat[\href{https://movingai.com/benchmarks/mapf/random-64-64-10.png}{\texttt{Random Map}}]{\includegraphics[width=0.5\linewidth]{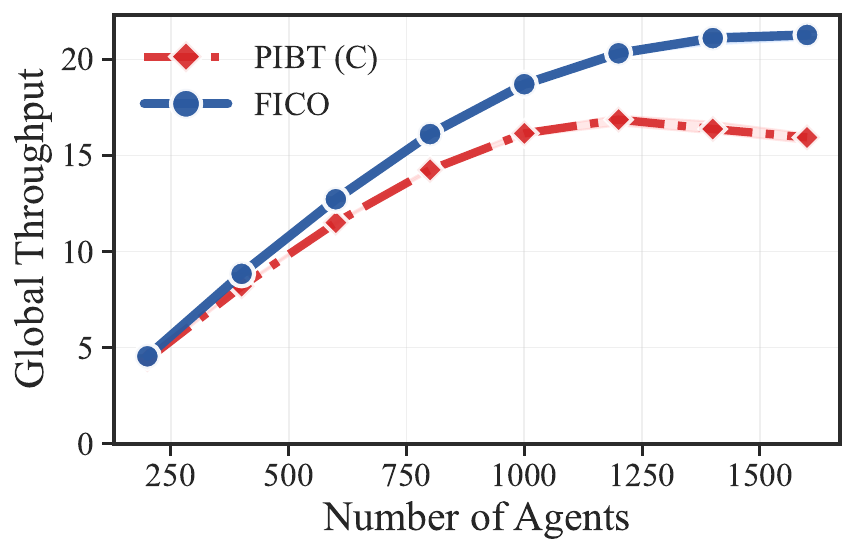}} \\
    \subfloat[\href{https://movingai.com/benchmarks/mapf/warehouse-20-40-10-2-1.png}{\texttt{Warehouse Map A}}]{\includegraphics[width=0.5\linewidth]{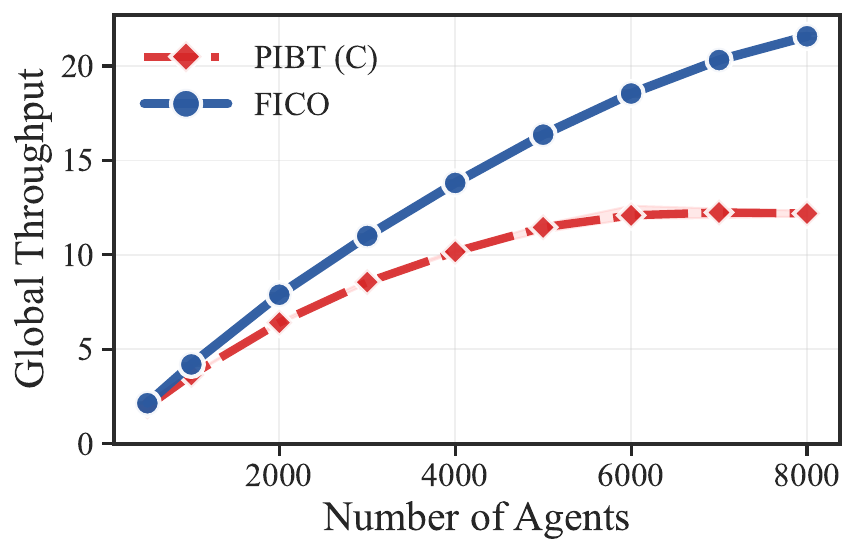}} 
    \subfloat[\href{https://movingai.com/benchmarks/mapf/warehouse-20-40-10-2-2.png}{\texttt{Warehouse Map B}}]{\includegraphics[width=0.5\linewidth]{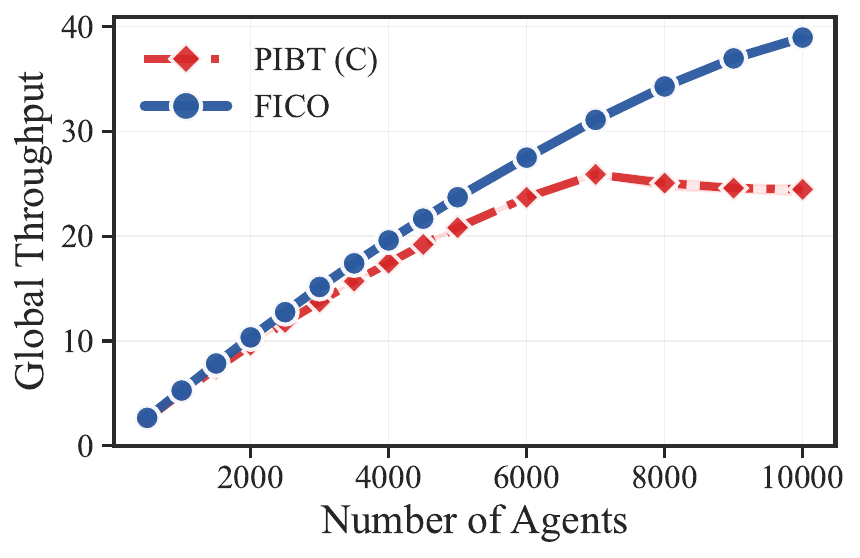}} 
    
    \caption{FICO consistently outperforms closed-loop baselines in terms of throughput in lifelong \gls{acr:mapf} (\cref{prob:searching-lifelong-mapf}). The performance increase is more significant in dense scenarios. }
    \label{fig:throughput_comparison_lifelong_mapf}
    \vspace{-5pt}
\end{figure}

\subsection{Ablation Study}
\label{exp:ablation-study}
To better understand the contribution of each component, we conduct a series of ablations: (i) removing parallel computation, (ii) replacing the balanced tie-breaker with a random one, (iii) varying the planning horizon, and (iv) removing the hindrance mechanism~\cite{okumura2025lightweight} and using the PIBT with the balanced tie-breaker in \cref{lem:optimal-path-distribution}.

\myparagraph{Parallel computation.}
\gls{acr:fico} exploits parallelization in step~\circled{1} (individual path planning) and step~\circled{4} (groupwise replanning).
Converting both to sequential execution significantly increases \gls{acr:ert}, particularly at larger scales (\cref{fig:ablation_multi_processing}).
The benefit is hardware-limited in our experiments (12 threads), but in principle scales linearly, as no communication is required across agents or groups.

\begin{figure}[tb]
    \centering
    \subfloat[\href{https://movingai.com/benchmarks/mapf/random-64-64-10.png}{\texttt{Random Map}}, $H=3$]{\includegraphics[width=0.5\linewidth]{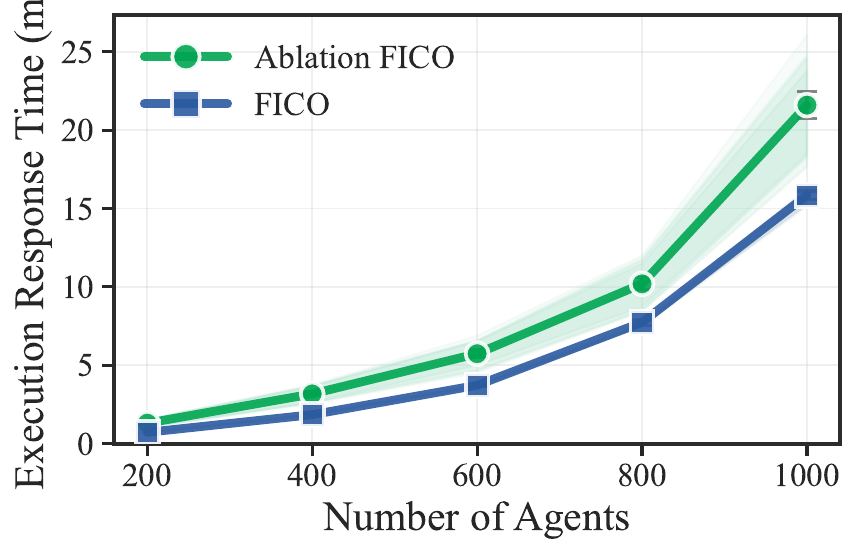}} \hfill
    \subfloat[\href{https://movingai.com/benchmarks/mapf/warehouse-20-40-10-2-2.png}{\texttt{Warehouse Map B}}, $H=3$]{\includegraphics[width=0.5\linewidth]{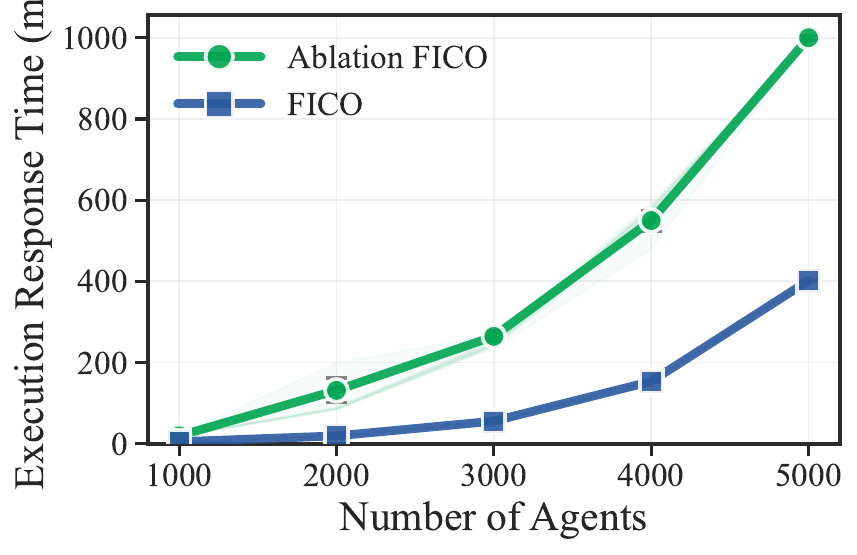}} 
    
    \caption{\textbf{Effect of the multi-processing mechanism.} Removing the multi-processing parallel computation mechanism in \anothercircled{1} and \anothercircled{4} will cause a significant slowdown across all cases, especially in scenarios with more agents.}
    \label{fig:ablation_multi_processing}
    \vspace{-5pt}
\end{figure}

\begin{figure}[tb]
    \centering
    \subfloat[\href{https://movingai.com/benchmarks/mapf/random-64-64-10.png}{\texttt{Random Map}}, one-shot]{\includegraphics[width=0.5\linewidth]{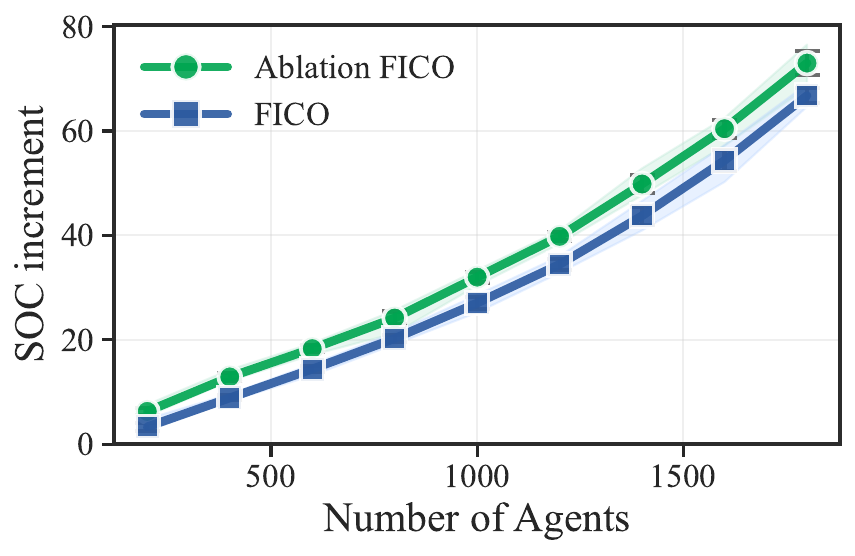}} \hfill
    \subfloat[\href{https://movingai.com/benchmarks/mapf/random-64-64-10.png}{\texttt{Random Map}}, lifelong]{\includegraphics[width=0.5\linewidth]{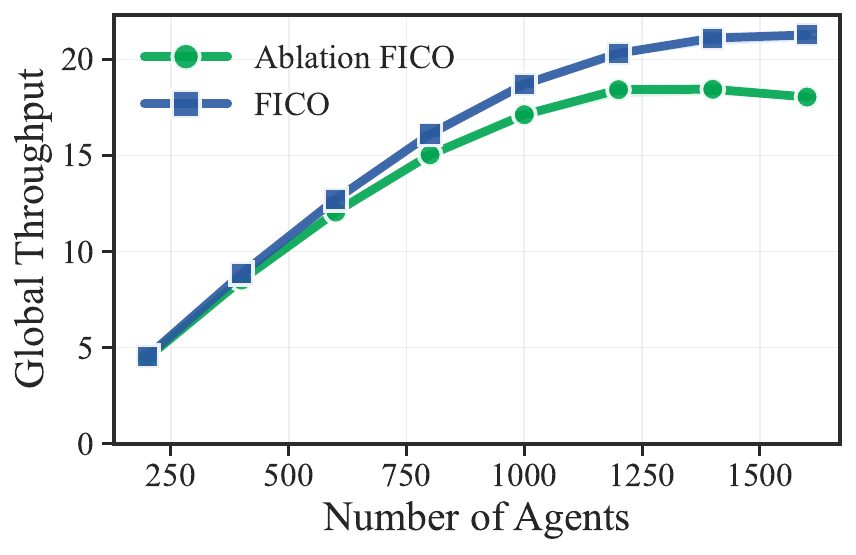}} \\
    \subfloat[\href{https://movingai.com/benchmarks/mapf/warehouse-20-40-10-2-2.png}{\texttt{Warehouse Map B}}, one-shot]{\includegraphics[width=0.5\linewidth]{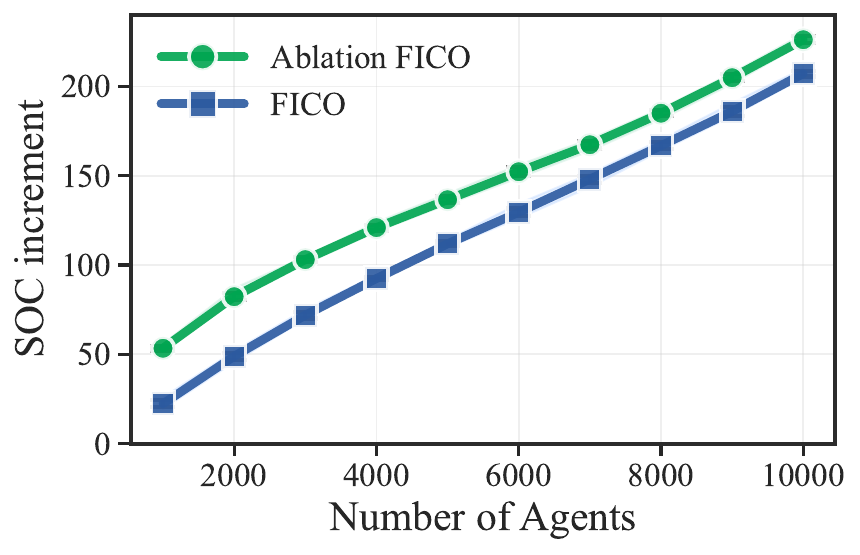}} \hfill
    \subfloat[\href{https://movingai.com/benchmarks/mapf/warehouse-20-40-10-2-2.png}{\texttt{Warehouse Map B}}, lifelong]{\includegraphics[width=0.5\linewidth]{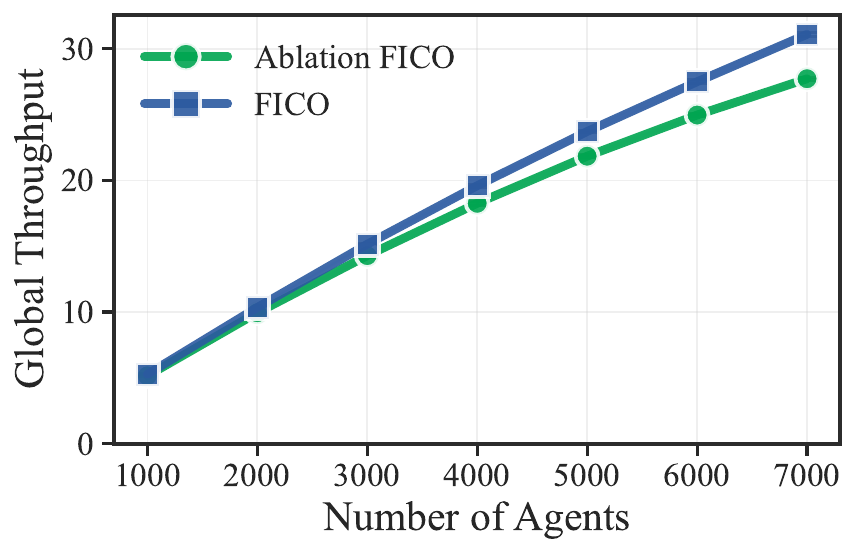}}
    \caption{\textbf{Effect of balanced vs.\ random tie-breaking.} Replacing the balanced tie-breaker in~\cref{lem:optimal-path-distribution} will cause notable solution quality degradation in both one-shot and lifelong \gls{acr:mapf} in different maps across all densities, while the effect is more significant in dense situations. }
    \label{fig:ablation_balanced_tie_breaker}
    \vspace{-10pt}
\end{figure} 

\myparagraph{Balanced tie-breaker.}
The balanced tie-breaker of \cref{lem:optimal-path-distribution} ensures uniform selection among individually optimal paths, avoiding the skew introduced by random tie-breaking.
Replacing it with a random rule leads to substantial \gls{acr:soc} and throughput degradation in the one-shot and lifelong settings (\cref{fig:ablation_balanced_tie_breaker}), confirming the importance of maintaining uniformity.

\myparagraph{Planning horizon.}
The horizon~$H$ controls the algorithm's lookahead and directly impacts coordination.
\cref{fig:ablation_alternative_horizon_size} shows the results of this ablation study, where we alternate the horizon parameter in the same experiment setting. 
In both the \href{https://movingai.com/benchmarks/mapf/random-64-64-10.png}{\texttt{Random Map}} and \href{https://movingai.com/benchmarks/mapf/warehouse-20-40-10-2-2.png}{\texttt{Warehouse Map B}}, when the number of agents is relatively small (\cref{subfig:horizon-random-small} and~\cref{subfig:horizon-warehouse-small}), the solution qualities increase monotonically with longer planning horizons, indicating that extended lookahead is beneficial in these settings. However, in relatively dense scenarios (\cref{subfig:horizon-random-large} and~\cref{subfig:horizon-warehouse-large}), the improvements diminish substantially beyond a certain threshold, represented by a ``Nike-shaped'' curve, reflecting the trade-off between short- and long-sightedness.

\begin{figure}[tb]
    \centering
    \subfloat[\href{https://movingai.com/benchmarks/mapf/random-64-64-10.png}{\texttt{Random Map}}, 200 agents]{\includegraphics[width=0.5\linewidth]{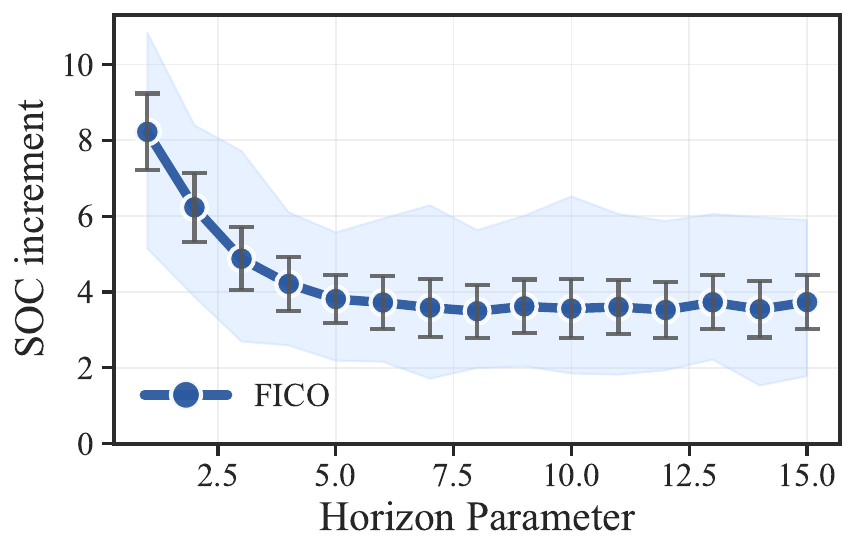}\label{subfig:horizon-random-small}} \hfill
    \subfloat[\href{https://movingai.com/benchmarks/mapf/random-64-64-10.png}{\texttt{Random Map}}, 600 agents]{\includegraphics[width=0.5\linewidth]{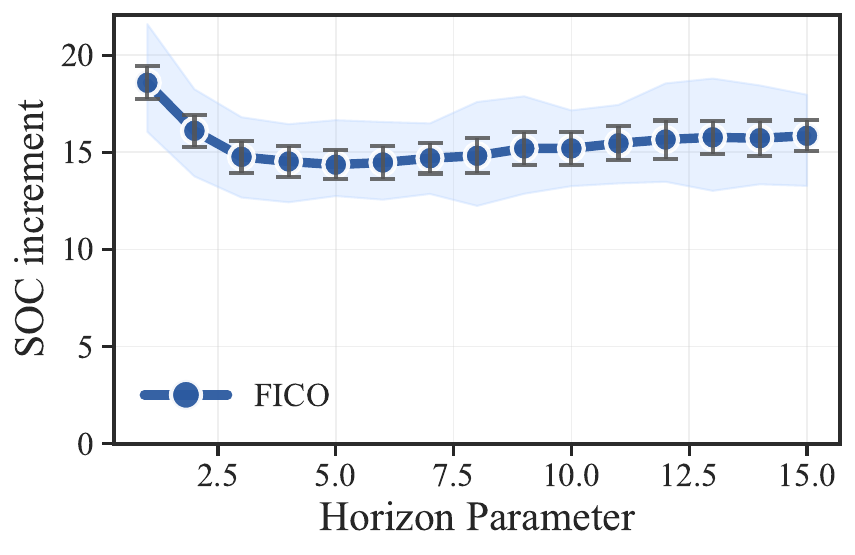}\label{subfig:horizon-random-large}} \\
    \subfloat[\href{https://movingai.com/benchmarks/mapf/warehouse-20-40-10-2-2.png}{\texttt{Warehouse Map B}}, 1000 agents]{\includegraphics[width=0.5\linewidth]{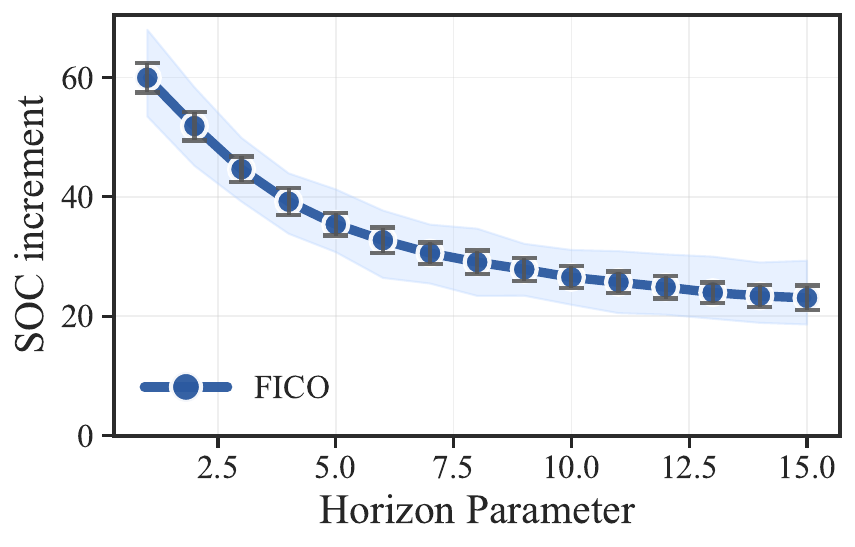}\label{subfig:horizon-warehouse-small}} \hfill
    \subfloat[\href{https://movingai.com/benchmarks/mapf/warehouse-20-40-10-2-2.png}{\texttt{Warehouse Map B}}, 5000 agents]{\includegraphics[width=0.5\linewidth]{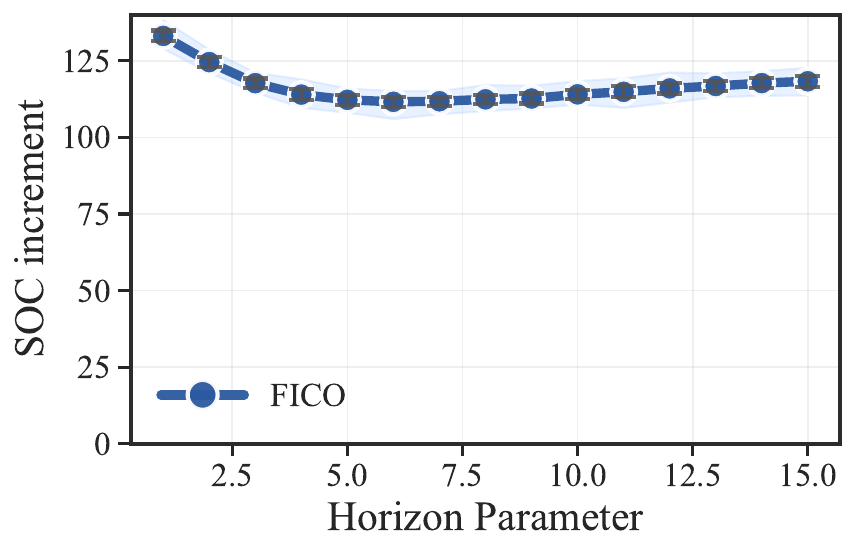}\label{subfig:horizon-warehouse-large}}
    \caption{\textbf{Effect of horizon parameter $H$.} As a decisive hyperparameter, changing the horizon will bring different solution qualities. }
    \label{fig:ablation_alternative_horizon_size}
    \vspace{-10pt}
\end{figure}  

\begin{figure}[tb]
    \centering
    \subfloat[\href{https://movingai.com/benchmarks/mapf/random-64-64-10.png}{\texttt{Random Map}}]{\includegraphics[width=0.5\linewidth]{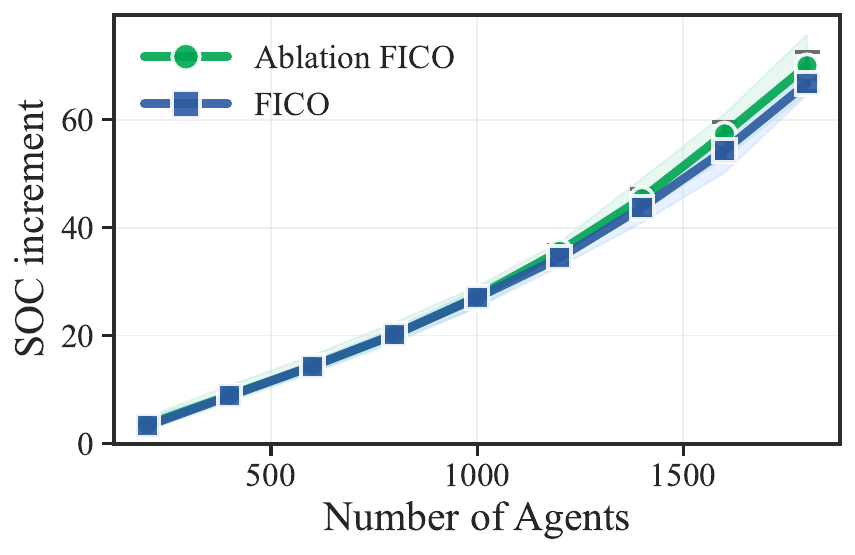}} \hfill
    \subfloat[\href{https://movingai.com/benchmarks/mapf/warehouse-20-40-10-2-2.png}{\texttt{Warehouse Map B}}]{\includegraphics[width=0.5\linewidth]{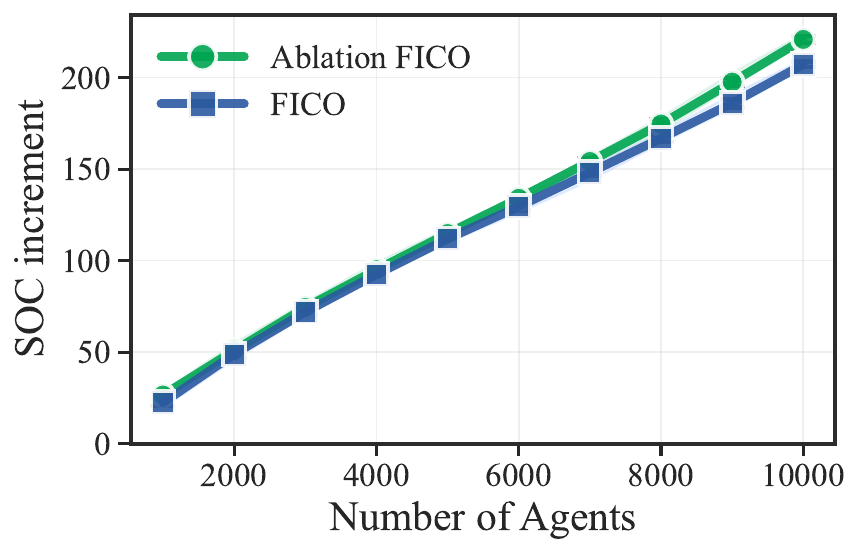}} 
    \caption{\textbf{Effect of the hindrance mechanism in the parallel factorization-compatible replanning~\anothercircled{4}.} Compared with the PIBT-based planner with the balanced tie-breaker, the hindrance mechanism can bring better performance, especially in dense scenarios. }
    \label{fig:ablation_hindrance}
    \vspace{-10pt}
\end{figure}  

\myparagraph{Hindrance.}
We next evaluate the effect of incorporating hindrance-enhanced PIBT in step~\circled{4}, where neighboring vertices are prioritized based on potential interference.
Hindrance-enhanced PIBT, introduced in~\cite{okumura2025lightweight}, augments the standard PIBT by computing a \emph{hindrance score} for each neighbor of an agent, reflecting the potential impact on other agents. 
During neighbor selection, the hindrance score serves as a secondary sorting key, while the primary key remains the distance to the goal. 
This modification replaces the vanilla PIBT approach, where ties are broken randomly, which often leads to suboptimal outcomes. 
An alternative to using hindrance scores is to employ the balanced tie-breaking strategy described in~\cref{lem:optimal-path-distribution} (and applied in step~\circled{1}), which balances neighbor selection among equally distant candidates. 
Our results show that substituting hindrance scoring with balanced tie-breaking slightly reduces solution quality in dense environments, while in less congested scenarios the two strategies perform comparably (\cref{fig:ablation_hindrance}).

\begin{figure}[tb]
    \centering
    \subfloat[\href{https://movingai.com/benchmarks/mapf/warehouse-20-40-10-2-2.png}{\texttt{Warehouse Map B}}, $p_{\mathrm{delay}} = p_{\mathrm{add}} = 0.1$]{\includegraphics[width=\linewidth]{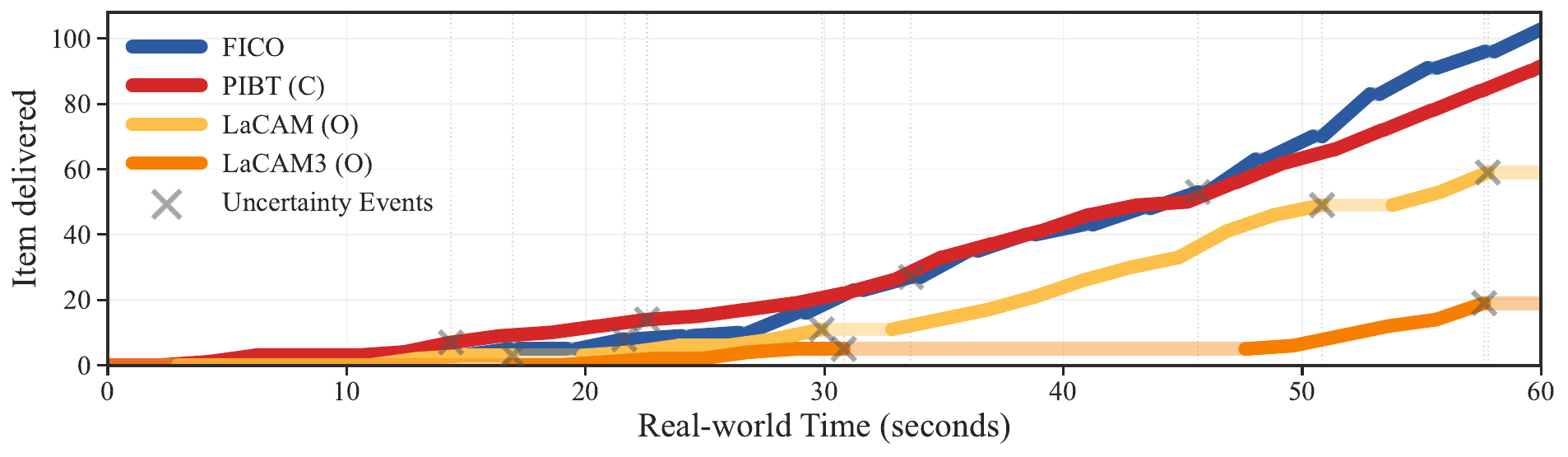}} \\
    \subfloat[\href{https://movingai.com/benchmarks/mapf/warehouse-20-40-10-2-2.png}{\texttt{Warehouse Map B}}, $p_{\mathrm{delay}} = p_{\mathrm{add}} = 0.5$]{\includegraphics[width=\linewidth]{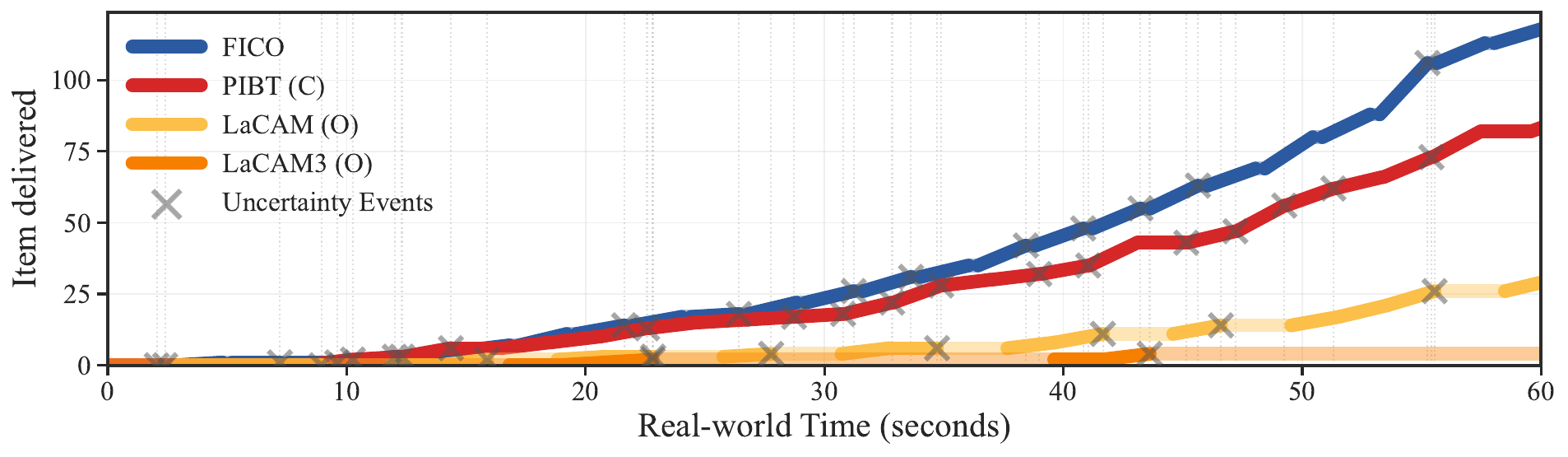}} \\
    \subfloat[\href{https://movingai.com/benchmarks/mapf/warehouse-20-40-10-2-2.png}{\texttt{Warehouse Map B}}, $p_{\mathrm{delay}} = p_{\mathrm{add}} = 0.8$]{\includegraphics[width=\linewidth]{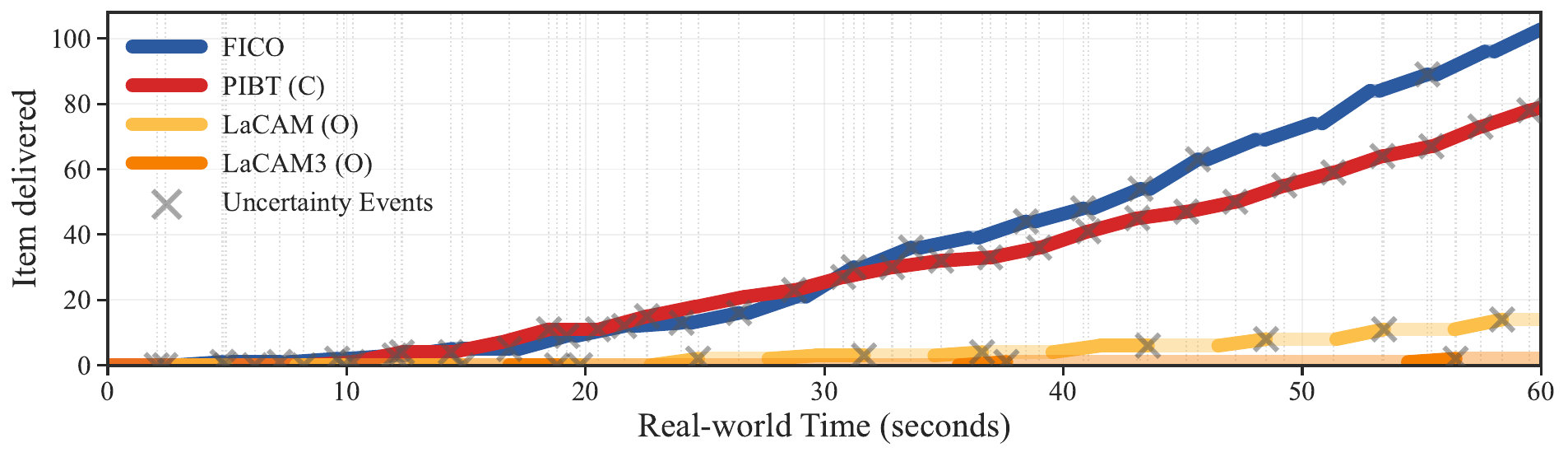}} 
    \caption{\textbf{Item delivery performance under different levels of uncertainty.} 
    Open-loop algorithms degrade sharply due to costly replanning whenever uncertainties occur. 
    The closed-loop baseline PIBT maintains fast planning but yields poor solution quality. 
    \gls{acr:fico} combines closed-loop adaptability with high-quality solutions, achieving the best performance.}
    \label{fig:uncertainties}
    \vspace{-10pt}
\end{figure}

\subsection{Robustness under uncertainty} \label{sec:experiments-uncertainties}
To assess robustness, we conduct a case study of warehouse item delivery in the presence of stochastic execution delays and dynamic agent arrivals.
These uncertainties are incorporated into the \gls{acr:mapf} framework by replacing the idealized actuators and environment with uncertain counterparts (cf.~\cref{exp:actuator-stochastic-delay,exp:env-agent-addition}), yielding variants of the unified \gls{acr:mapf} problem where execution is no longer deterministic. 
We deploy $5000$ agents in the \href{https://movingai.com/benchmarks/mapf/warehouse-20-40-10-2-2.png}{\texttt{Warehouse Map B}} and measure the number of items delivered within one minute. 
Each execution step is assumed to take $2$ seconds, and robots cannot plan while moving, since uncertainties are only revealed once a movement is completed.
\cref{fig:uncertainties} reports the results. 
The light lines indicate planning trajectories, while the dark lines denote actual execution. 
Open-loop baselines, though capable of high solution quality in deterministic settings (cf.~\cref{fig:soc_comparison_oneshot_mapf}), become bottlenecked by repeated, time-consuming replanning in uncertain environments. 
In contrast, the closed-loop baseline PIBT avoids this bottleneck but produces low-quality solutions due to its greedy nature. 
By leveraging factorization and adaptivity, \gls{acr:fico} achieves the highest throughput. 
For instance, when~$p_{\mathrm{delay}} = p_{\mathrm{add}} = 0.5$, \gls{acr:fico} is able to deliver $35.7\%$ more items than the closed-loop baseline, and improves over the open-loop baselines by $353.5\%$ and $3175\%$, respectively.

\section{Conclusion and Future Work}
\label{sec:conclusion}

This paper introduced the \gls{acr:mapf} system, a unified execution-aware model that integrates planning and execution within a feedback loop, and proposed FICO, a closed-loop factorization algorithm inspired by receding-horizon control. 
By combining finite-horizon lookahead with dual-level factorization, FICO leverages the compositional structure of \gls{acr:mapf} to achieve scalability, responsiveness, and robustness. 
Extensive experiments across one-shot, lifelong, and uncertainty-aware settings show that FICO consistently delivers highly competitive solution quality while maintaining real-time responsiveness. 

Looking forward, the framework opens several avenues for exploration. 
First, \gls{acr:fico}’s modular structure enables integration with more sophisticated planning algorithms, providing a pathway to FICO+, which can transform existing \gls{acr:mapf} algorithms to closed-loop algorithms for the unified \gls{acr:mapf} problem. 
Second, the system-level modeling of \gls{acr:mapf} and the controller design perspective of the \gls{acr:mapf} problem can guide the design of new benchmarks, open-loop and closed-loop, and evaluation metrics that explicitly account for feedback, uncertainty, and continual operation. 
Finally, applying FICO to broader domains such as multi-robot manipulation or mixed autonomy traffic systems offers exciting opportunities~\cite{zardini2022analysis}.


{
  \bibliographystyle{IEEEtran}
  \bibliography{references}
}

\vspace{-11pt}
\appendix

\subsection{Assistive algorithms used in \gls{acr:fico}}

\begin{algorithm}[h]
\resizebox{\columnwidth}{!}{
\begin{minipage}{1\columnwidth}
    \small
    \caption{Perfect distance heuristic computation}
    \label{alg:perfect_distance_heuristic}
    \begin{algorithmic}[1]
    \Require{Instance $\cI_t = \tup{G^t,A^t,\rho_s^t, \rho_g^t}$, agent $a_i \in A^t$, vertex $v \in V^t$, pre-computed table $\gamma \colon V^t \times A^t \to \mathbb{N}$ which stores the computed distances and initialized with $\infty$, and queue $Q_d$ which stores the vertices to be visited}
    \Ensure{The distance from $v$ to $\rho_g^t(a_i)$, denoted as $\gamma(v,a_i)$}

    \If{$\gamma(v,a_i) \neq \infty$}
        \State \Return{$\gamma(v,a_i)$}
    \EndIf
    \If{$Q_d = \emptyset$}
        \State $Q_d.\textbf{push}(\rho_g^t(a_i))$
        \State $\gamma(v,a_i) \gets 0$
    \EndIf
    \While{$Q_d \neq \emptyset$}
        \State $v^\prime \gets$ \textbf{pop}($Q_d$)
        \State $d_{v^\prime} \gets \gamma(v^\prime,a_i)$
        \ForAll{$\tilde{v} \in \mathrm{Neighbors}(v^\prime)$}
            \State $\gamma(\tilde{v},a_i) \gets d_{v^\prime} + 1$
            \State $Q_d.\textbf{push}(\tilde{v})$
        \EndFor
        \If{$v^\prime = v$}
            \State \Return{$\gamma(v,a_i)$}
        \EndIf
    \EndWhile
    \end{algorithmic}
\end{minipage}
}
\end{algorithm}
\begin{algorithm}[h]
    \resizebox{\columnwidth}{!}{
    \begin{minipage}{1\columnwidth}
    \small
    \caption{Optimal path count calculation}
    \label{alg:optimal_path_count}
    \begin{algorithmic}[1]
    \Require{Instance $\cI_t = \tup{G^t,A^t,\rho_s^t, \rho_g^t}$, agent $a_i \in A^t$, vertex $v \in V^t$, pre-computed table $C \colon V^t \times A^t \to \mathbb{N}$ which stores the computed optimal path counts and initialized with $0$, and priority queue $PQ_c$ which stores the vertices to be visited according to their distances to the goal $\gamma(v,a_i)$}
    \Ensure{The optimal path count from $v$ to $\rho_g^t(a_i)$, denoted as $c(v,a_i)$}

    \If{$C(v,a_i) \neq 0$}
        \State \Return{$C(v,a_i)$}
    \EndIf
    \If{$PQ_c = \emptyset$}
        \State $PQ_c.\textbf{push}(\rho_g^t(a_i))$
        \State $C(v,a_i) \gets 1$
        \ForAll{$\tilde{v} \in \mathrm{Neighbors}(\rho_g^t(a_i))$}
            \State $PQ_c.\textbf{push}(\tilde{v})$
        \EndFor
    \EndIf
    \While{$PQ_c \neq \emptyset$}
        \State $v^\prime \gets$ \textbf{pop}($PQ_c$)
        \State $c_{v^\prime} \gets 0$
        \ForAll{$\tilde{v} \in \mathrm{Neighbors}(v^\prime)$}
            \If{$\gamma(\tilde v,a_i) = \gamma(v',a_i)-1$}
                \State $c_{v^\prime} \gets c_{v^\prime} + C(\tilde{v},a_i)$
            \ElsIf{$\gamma(\tilde v,a_i) = \gamma(v',a_i)+1$}
                \State $PQ_c.\textbf{push}(\tilde{v})$
            \EndIf
        \EndFor
        \State $C(v^\prime,a_i) \gets c_{v^\prime}$

        \If{$v^\prime = v$}
            \State \Return{$C(v,a_i)$}
        \EndIf
    \EndWhile
    \end{algorithmic}
    \end{minipage}
    }
\end{algorithm}

\end{document}